\documentclass{article}

\usepackage{microtype}
\usepackage{graphicx}
\usepackage{booktabs} 

\usepackage{subcaption}
\usepackage{multicol, multirow}

\usepackage{amsmath, amssymb, mathtools, amsthm}
\usepackage{tabulary}
\usepackage{wrapfig, makecell}
\usepackage{enumitem}
\usepackage{pifont}
\usepackage{colortbl}
\usepackage{xspace}
\usepackage{tabulary}
\usepackage{wrapfig, makecell}
\usepackage{enumitem}
\usepackage{pifont}
\usepackage{bm}
\usepackage[table, svgnames, dvipsnames]{xcolor}

\usepackage{hyperref}

\usepackage[accepted]{icml2024}
\newcommand{\alg}{\textbf{\textsc{DistiLLM}}\xspace}

\definecolor{pinegreen}{rgb}{0.0, 0.47, 0.44}
\usepackage{enumitem}
\usepackage{tcolorbox}

\theoremstyle{plain}
\newtheorem{theorem}{Theorem}

\newtheorem{remark}{Remark}
\newtheorem{assumption}{Assumption}[section]
\newtheorem{lemma}{Lemma}[section]

\newcommand{\mytitle}{\alg: Towards Streamlined Distillation for Large Language Models}

\icmltitlerunning{\mytitle}

\begin{document}
\definecolor{myred}{RGB}{204, 4, 67}
\definecolor{darkpastelgreen}{rgb}{0.01, 0.75, 0.24}
\definecolor{electriccrimson}{rgb}{1.0, 0.0, 0.25}
\newcommand{\greenup}{\textbf{\textcolor{darkpastelgreen}{$\bullet$}}}
\newcommand{\reddown}{\textbf{\textcolor{electriccrimson}{$\bullet$}}}

\twocolumn[
\icmltitle{\mytitle}




\begin{icmlauthorlist}
\icmlauthor{Jongwoo Ko}{kaist}
\icmlauthor{Sungnyun Kim}{kaist}
\icmlauthor{Tianyi Chen}{ms}
\icmlauthor{Se-Young Yun}{kaist} 
  \vspace{2pt} \\
\url{https://github.com/jongwooko/distillm}
\end{icmlauthorlist}

\icmlaffiliation{kaist}{KAIST AI, Seoul, Republic of Korea}
\icmlaffiliation{ms}{Microsoft, Redmond, Washington, USA}

\icmlcorrespondingauthor{Se-Young Yun}{yunseyoung@kaist.ac.kr}

\icmlkeywords{Machine Learning, ICML}

\vskip 0.3in
]



\printAffiliationsAndNotice{}  

\begin{abstract}
Knowledge distillation (KD) is widely used for compressing a teacher model to a smaller student model, reducing its inference cost and memory footprint while preserving model capabilities. However, current KD methods for auto-regressive sequence models (\textit{e.g.}, large language models) suffer from missing a standardized objective function. Moreover, the recent use of student-generated outputs to address training-inference mismatches has significantly escalated computational costs. To tackle these issues, we introduce \alg, a more effective and efficient KD framework for auto-regressive language models. \alg comprises two components: (1)\,a novel skew Kullback-Leibler divergence loss, where we unveil and leverage its theoretical properties, and (2)\,an adaptive off-policy approach designed to enhance the efficiency in utilizing student-generated outputs. Extensive experiments, including instruction-following tasks, demonstrate the effectiveness of \alg in building high-performing student models while achieving up to 4.3$\times$ speedup compared to recent KD methods.
\vspace{-15pt}
\end{abstract}


\section{Introduction}

Recent advancements in auto-regressive language models (LMs, \citealt{openai2023gpt4, touvron2023llama2}) such as large language models\,(LLMs) have significantly improved the quality of text generation in a variety of generative tasks such as task-agnostic instruction-following tasks\,\cite{wang-etal-2023-self-instruct}. This success is often attributed to the increased scale of training data and model parameters\,(\textit{e.g.}, 175B parameters for GPT-3; \citealt{brown2020gpt3}). However, expanding the parameter count brings associated costs, limiting the deployment of these models due to either high inference costs or large memory footprints. Therefore, a crucial objective for the practical use of these high-capacity models is to compress them by reducing the number of parameters while preserving their performance to the greatest extent possible.


As the necessity of reducing the demands on computational resources becomes important, KD\,\cite{Hinton2015DistillingTK} emerges as a promising method. It involves the transfer of knowledge from a teacher model with large capacity\,(\textit{i.e.,} LLMs) to a smaller student LM.
Most approaches of KD have employed the Kullback-Leibler divergence\,(KLD) loss, enforcing the student model to generate outputs that mirror the teacher model's outputs on a fixed dataset\,\cite{kim-rush-2016-sequence, Sanh2019DistilBERTAD}.
Such methods of KD have significantly enhanced the performance of student models, making them competitive with teacher models while increasing efficiency, especially for various classification tasks\,\cite{sun-etal-2019-patient, mirzadeh2020improved}.

These approaches using KLD on fixed datasets, however, have two primary shortcomings in applying to auto-regressive LMs.
First, using the KLD can lead to sub-optimal results. Given the complexity of generative tasks compared to classification tasks, this can result in the student distribution becoming overly smooth and consequently failing to fully capture the teacher distribution, or becoming overly concentrated in high-probability distributions. This issue---referred to as \textit{mode averaging} or \textit{mode collapse}---arises due to the asymmetric nature of the KLD\,\cite{wen-etal-2023-f, gu2023knowledge}.
Second, the use of a fixed dataset in the training phase can cause a distribution mismatch between the sequences observed during training and those generated by the student in auto-regressive inference, leading to exposure bias problems\,\cite{arora-etal-2022-exposure}.

Recent studies have explored various divergence losses \cite{wen-etal-2023-f, agarwal2023gkd} or the incorporation of student-generated output (SGO, \citealt{lin-etal-2020-autoregressive, agarwal2023gkd}) to address the existing problems. However, these methods often lack standardized objective functions and are less efficient due to the continuous SGO generation. 
For instance, \citet{agarwal2023gkd} employed on-policy distillation with SGOs, but it faces low sample efficiency and high generation time as it constantly prompts the student to produce new training sequences.
Also, their experiments indicated the optimal divergence seems to be task-dependent, which requires additional efforts to inconveniently select a proper loss function.
\citet{gu2023knowledge} introduced a policy optimization method to minimize reverse KLD between student and teacher distributions, yet this also compromises training efficiency by requiring generation from both models in every iteration.

\vspace{-7.5pt}
\paragraph{Contributions.}
In this paper, we introduce \alg\footnote{\alg is pronounced as \emph{distill-LLM}, merging the word ``distill'' with ``LLM''.}, featuring a novel \textit{skew KLD loss} and an \textit{adaptive off-policy approach}, focusing on both distillation effectiveness and training efficiency. We provide both theoretical and empirical evidence that the components of \alg work well individually and synergistically with each other. Our detailed contributions include:
\begin{itemize}[leftmargin=*, itemsep=0pt]
\vspace{-10pt}
    \item \textbf{Skew KLD:} We focus on the two key issues of existing objective functions for auto-regressive LMs: instability from potential gradient explosions in optimizing the KLD loss, and lack of emphasis on generalizability and convergence. To address these limitations, we introduce skew KLD, a new objective function with a strong theoretical foundation, optimized for stable gradients and minimal approximation errors, empirically leading to faster convergence and superior performance.
    \item \textbf{Adaptive off-policy approach:} While using SGOs in KD is generally effective in improving performance, this approach significantly increases training time~(Fig.~\ref{fig:fig_time}) and makes it challenging to find the optimal proportion for using SGOs. To this end, we propose an \textit{adaptive off-policy} approach module for adaptively and efficiently leveraging SGOs to consider the data perspective of KD. 
    \item \textbf{Advanced performance and efficiency:} \alg accomplishes state-of-the-art performances for the student LMs on various generative tasks (\textit{e.g.,} instruction-following or text summarization), while achieving the 2.5\,$\sim$\,4.3$\times$ training speedup compared to recent KD techniques\,\cite{gu2023knowledge, agarwal2023gkd}.
\end{itemize}
\vspace{-3pt}
\section{Background}
\vspace{-2pt}

\subsection{KD for Auto-regressive Generative LMs}


We provide preliminary information on the KD for auto-regressive generative LMs. Given a source and target sequence pair, denoted as $(\mathbf{x}, \mathbf{y})$, KD minimizes divergence $D$ between the distributions of a fixed teacher model $p(\mathbf{y} | \mathbf{x})$ and a parameterized student model $q_{\theta}(\mathbf{y} | \mathbf{x})$. The training data pairs $(\mathbf{x}, \mathbf{y})$ are either sampled from a fixed ground-truth dataset\,\cite{Hinton2015DistillingTK} or from teacher-generated outputs\,\cite{kim-rush-2016-sequence}.

Conventionally, KLD, denoted as ${D}_{\text{KL}}$, is the most widely used loss in KD due to its simplicity and tractability. The sequence-level distillation using the KLD is accurately decomposed into a sum of token-wise distillation:
\vspace{-5pt}
\begin{align}
    \textstyle &{D}_{\text{KL}}(p, q_\theta) 
    = \mathbb{E}_{\mathbf{x}} \mathbb{E}_{\mathbf{y} \sim p(\cdot|\mathbf{x})} \left[ \textstyle \log \frac{p(\mathbf{y}|\mathbf{x})}{q_{\theta}(\mathbf{y}|\mathbf{x})} \right] \label{eq:kld} \\
    &\approx \textstyle \frac{1}{|\mathcal{D}|} \sum_{(\mathbf{x}, \mathbf{y}) \in \mathcal{D}} p(\mathbf{y}|\mathbf{x})  \log \frac{p(\mathbf{y}|\mathbf{x})}{q_{\theta}(\mathbf{y}|\mathbf{x})} \label{eq:approx_kld} \\
    &= \textstyle \frac{1}{|\mathcal{D}|} \sum_{\mathbf{x}\in\mathcal{D}_{X}}\sum_{t}^{|\mathbf{y}|} \sum_{y_{t} \in V} p(y_{t}|\mathbf{y}_{<t}, \mathbf{x}) \log \frac{p(y_{t}|\mathbf{y}_{<t}, \mathbf{x})}{q_{\theta}(y_{t}|\mathbf{y}_{<t}, \mathbf{x})} \label{eq:token}
\end{align}
where $V$ is the vocabulary token set and $\mathbf{y}_{<t} \coloneqq (y_{1}, y_{2}, \ldots, y_{t-1})$ represents the sequence of tokens up to index $t-1$. We focus solely on tractable KLD, as other divergences like total variation distance~(TVD, \citealt{wen-etal-2023-f}) do not effectively decompose sequence-level distillation into token-level components.
While the explicit definition of KLD is given in Eq.\,\ref{eq:kld} \cite{kim-rush-2016-sequence, wen-etal-2023-f}, most recent studies, such as \citet{agarwal2023gkd} and \citet{gu2023knowledge}, approximate the distribution matching by minimizing Eq.\,\ref{eq:approx_kld}, under the assumption that the teacher's distribution is similar to its training dataset $\mathcal{D}$. For the sake of training efficiency, our method utilizes the definition provided in Eq.\,\ref{eq:approx_kld}, while Eq.\,\ref{eq:kld} is used for theoretical analysis of our proposed distillation objective in Thm.~\ref{method:thm}.

\subsection{Pitfalls of Existing Distillation}\label{sec:revisit}
\paragraph{Limitation of objective functions.}

The KLD objective in KD, primarily due to its asymmetric nature\,\cite{wen-etal-2023-f}, often forces the student distribution to cover the entire support set of the teacher distribution, leading to significant limitations.
This becomes evident when a sampled data point is included in the teacher distribution's support but falls outside the student distribution, \textit{i.e.}, $\exists(\mathbf{x}, \mathbf{y})$ such that $p(\mathbf{y} | \mathbf{x})\gg0$ and $q_\theta(\mathbf{y} | \mathbf{x}) \approx 0$. The limitation becomes pronounced if the student model lacks the capacity to match all support sets of the teacher distribution accurately. Consequently, this results in the student model exhibiting a mode-averaging problem, where it learns an overly smooth distribution in an attempt to cover the teacher's entire support set, as highlighted by recent studies\,\cite{wen-etal-2023-f, gu2023knowledge}.

Such recent studies have partially addressed this issue by applying the reverse KLD (RKLD, \citealt{gu2023knowledge, agarwal2023gkd}), defined as ${D}_{\text{RKL}}(p, q_\theta):= {D}_{\text{KL}}(q_\theta, p)$, or generalized JSD\,\cite{agarwal2023gkd} by introducing an interpolation parameter $\beta \in [0, 1]$, defined as
\vspace{-5pt}
\begin{equation}\label{eq:jsd}
\begin{split}
    D_{\text{JSD}}^{(\beta)}(p, q_{\theta}) &:= \beta\, D_{\text{KL}}(p, \beta p + (1-\beta) q_{\theta}) \\
    &+ (1-\beta)\, D_{\text{KL}}(q_{\theta}, \beta p + (1-\beta) q_{\theta}).
\end{split}
\end{equation}
These approaches have shown empirical success in auto-regressive LMs, but there is a need for systematic study to provide a standard distillation objective grounded in comprehensive theoretical and experimental analyses. The lack of such backing for these recently proposed objective functions leads to sub-optimal performance and task-dependent variability\,\cite{agarwal2023gkd}.

\vspace{-7.5pt}
\paragraph{Limitations of utilizing SGO.}
\begin{figure}[t]
    \centering
    \includegraphics[width=1.0\linewidth]{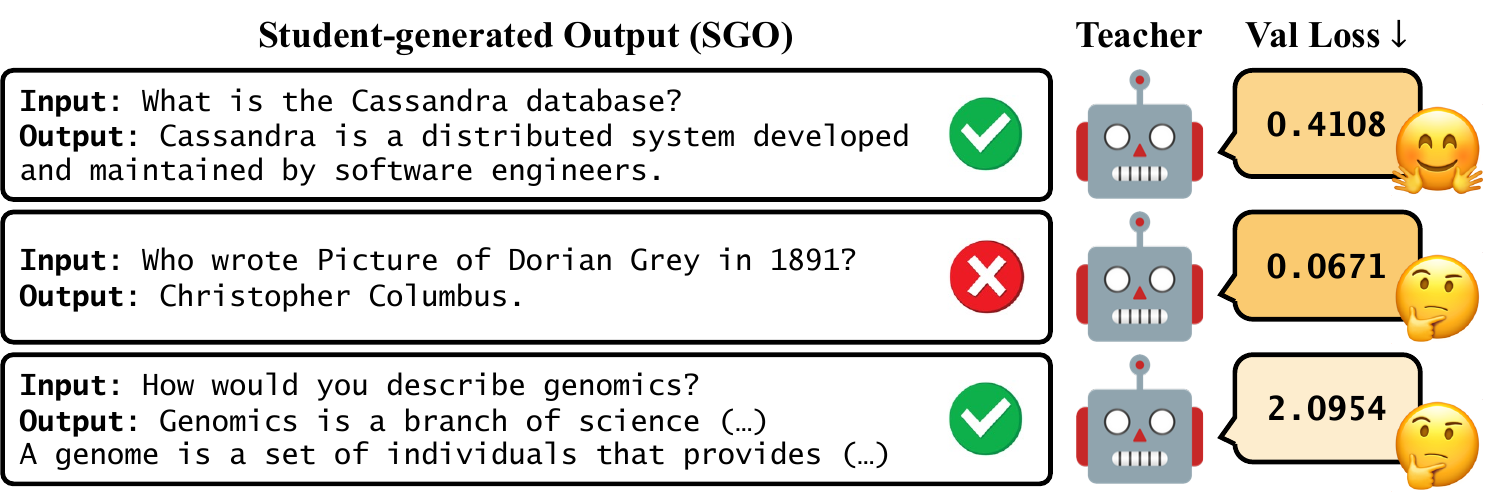}
    \vspace{-20pt}
    \caption{Examples of SGOs from GPT-2\,(student) and their corresponding validation loss by GPT-2 XL\,(teacher). 
    Since the teacher model may not be familiar with the SGO, using $p(\mathbf{y} | \mathbf{x})$ as a target distribution can misguide the student model, as shown in Tab.~\ref{tab:genfilt}.}
    \label{fig:ill_mismatch}
    \vspace{-5pt}
\end{figure}

\begin{figure}[t]
    \centering
    \includegraphics[width=1.0\linewidth]{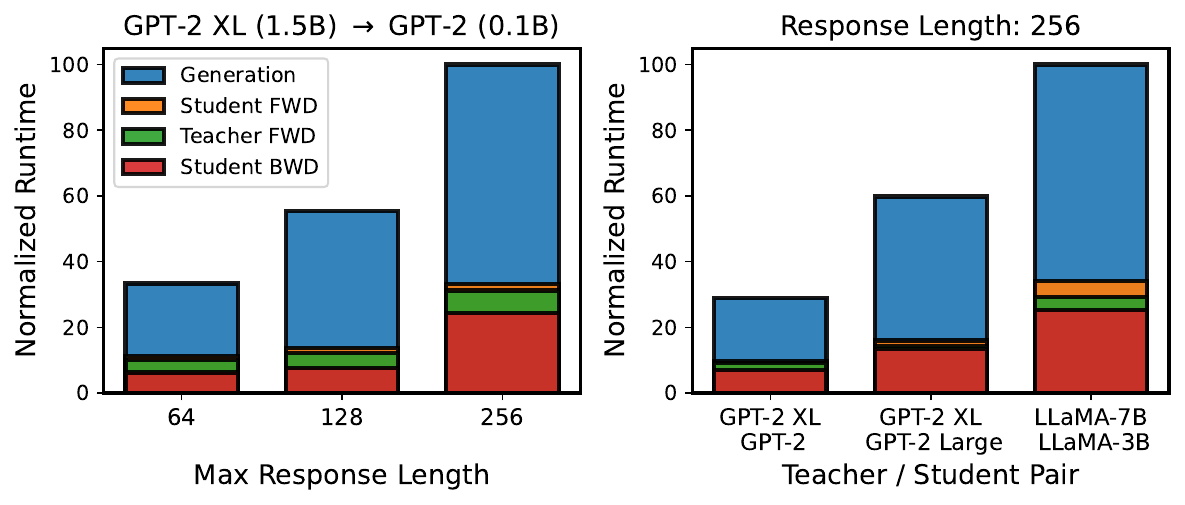}
    \vspace{-25pt}
    \caption{(\textbf{Left}): Normalized runtime according to the maximum response length of SGOs with GPT-2 XL teacher and GPT-2 student. (\textbf{Right}): Normalized runtime for various sizes of teacher and student models with a response length of 256. FWD and BWD denote forward and backward propagation, respectively.}
    \label{fig:fig_time}
    \vspace{-10pt}
\end{figure}
Previous KD methods for auto-regressive LMs have encountered a training-inference mismatch between the samples from fixed datasets that are used during training and those produced by the student model during inference. Recent studies\,\cite{lin-etal-2020-autoregressive, agarwal2023gkd} explored addressing this challenge by prompting the student model to generate SGOs and then training from the feedback of the teacher model on these sequences. This approach addresses the mismatch by training the student model on its familiar, self-generated sequences. These efforts have significantly improved the performance of distilling LLMs\,\cite{lin-etal-2020-autoregressive, agarwal2023gkd}.

Despite its effectiveness, we identify two main issues with the current utilization of SGO. \textit{First}, teacher models may experience a distribution mismatch between their training data and unfamiliar or inaccurate SGOs, potentially leading to misguidance on $q_{\theta}$. As depicted in Fig.\,\ref{fig:ill_mismatch}, such a mismatch can result in the teacher model assigning low validation loss to incorrect but shorter generations and high validation loss to longer but correct ones.
\textit{Second}, as shown in Fig.~\ref{fig:fig_time}, generating SGOs for every iteration proves computationally inefficient. Across all experiments, irrespective of the maximum sequence length~(ranging from 64 to 256) of SGOs or the model size~(from GPT-2 to OpenLLaMA2-3B), the SGO generation accounts for a considerable portion of the total training time, reaching up to 80\%.

However, to the best of our knowledge, there has been limited comprehensive effort to address these challenges simultaneously. For instance, MiniLLM\,\cite{gu2023knowledge} suggests a method that mixes the distributions of teacher and student to alleviate the first challenge. However, this method notably increases training computation due to the requirement of a large teacher model. These challenges motivate us to develop an approach that adaptively balances the positive effect of reduction of training-inference mismatch\,\cite{agarwal2023gkd} and the negative effect of performance degradation from noisy feedback~(as shown in Tab.~\ref{tab:genfilt}). Meanwhile, we also aim to improve the sample efficiency of SGO, thereby enhancing computational efficiency.

\begin{algorithm}[tb]
   \caption{Training pipeline of \alg}\label{alg:aesop}
\begin{algorithmic}[1]
   \STATE {\bfseries Input:} initial prob. $\phi$, student $q_{\theta_0}$ with parameters $\theta_0$, teacher $p$, total training iterations $T$, training \& validation dataset $\mathcal{D}$, $\mathcal{D}_{val}$, empty replay buffer $\mathcal{D}_{R}$
   \STATE {\bfseries Output:} Student model $q_{\theta_T}$ with trained parameters $\theta_T$
   \WHILE{$t \leq T$}
   \STATE Randomly sample $u \sim \text{Unif}(0, 1)$
   \STATE \textcolor{gray}{\textbf{\textit{/* Linearly Decreasing Replay Ratio */}}}
   \IF{$u < \lambda_{R} := \phi (1 - \frac{t}{T})$}
   \STATE \textcolor{gray}{\textit{\textbf{/* Generate SGO \& Update $\mathcal{D}_{R}$ */}}}
   \STATE Generate SGO $\{\tilde{\mathbf{y}}_{i}\}_{i=1}^{B}$ from $\{q_{\theta_t}(\cdot|\mathbf{x}_{i})\}_{i=1}^{B}$
   \STATE Store SGO into $\mathcal{D}_{R}$;~ $\mathcal{D}_{R} \leftarrow \mathcal{D}_{R} \cup \{(\mathbf{x}_{i}, \tilde{\mathbf{y}}_{i})\}_{i=1}^{B}$
   \ENDIF
   \IF{$u < \phi$}
   \STATE \textcolor{gray}{\textbf{\textit{/* Use SGO in Off-policy Approach (Fig.~\ref{fig:overview}(c))}}}
   \STATE Sample mini-batch $\{(\mathbf{x}_{i}, \tilde{\mathbf{y}}_{i})\}_{i=1}^{B}$ from $\mathcal{D}_{R}$
   \ELSE
   \STATE \textcolor{gray}{\textbf{\textit{/* Use Sample from Fixed Dataset (Fig.~\ref{fig:overview}(a)) */}}}
   \STATE Sample mini-batch $\{(\mathbf{x}_{i}, \mathbf{y}_{i})\}_{i=1}^{B}$ from $\mathcal{D}$
   \ENDIF
   \STATE \textcolor{gray}{\textbf{\textit{/* Use S(R)KL */}}}
   \STATE Update $\theta_t$ by S(R)KL $D_{\text{SKL}}^{(\alpha)}(\cdot, \cdot)$
   \IF{do validation}
   \STATE $\mathcal{L}_{prev}, \phi \leftarrow$ \texttt{SGO\_Scheduler}($\mathcal{L}_{prev}$, $\mathcal{D}_{val}$, $q_{\theta_{t}}$)
   \ENDIF 
   \ENDWHILE
   \STATE
   \STATE \textcolor{gray}{\textbf{/* \textit{Adaptive SGO Scheduler} */}}
   \STATE \textbf{def} \texttt{SGO\_Scheduler}($\mathcal{L}_{\tilde{t}-1}$, $\mathcal{D}_{val}$, $q_{\theta}$):
   \STATE \hskip1.5em \textcolor{gray}{\textbf{/* \textit{Compute Loss for Validation Set} */}}
   \STATE \hskip1.5em $\mathcal{L}_{\tilde{t}} \leftarrow \frac{1}{|D_{val}|} \sum_{\mathbf{x}_{val}, \mathbf{y}_{val}}$ \texttt{Loss}($q_\theta, \mathbf{x}_{val}, \mathbf{y}_{val}$)
   \STATE \hskip1.5em \textbf{if} $\mathcal{L}_{\tilde{t}} > \mathcal{L}_{\tilde{t}-1} + \varepsilon$ \textbf{then}
   \STATE \hskip1.5em \hskip1.5em Update $\phi_{\tilde{t}} \leftarrow \min(\phi_{\tilde{t}-1} + 1/N_{\phi}, 1.0)$
   \STATE \hskip1.5em \textbf{else}
   \STATE \hskip1.5em \hskip1.5em $\mathcal{L}_{\tilde{t}}, \phi_{\tilde{t}} \leftarrow \mathcal{L}_{\tilde{t}-1}, \phi_{\tilde{t}-1}$
   \STATE \hskip1.5em \textbf{end if}
   \STATE \hskip1.5em \textbf{return} $\mathcal{L}_{\tilde{t}}$, $\phi_{\tilde{t}}$
\end{algorithmic}
\end{algorithm}
\begin{figure*}[t]
    \centering
    \includegraphics[width=1.0\linewidth]{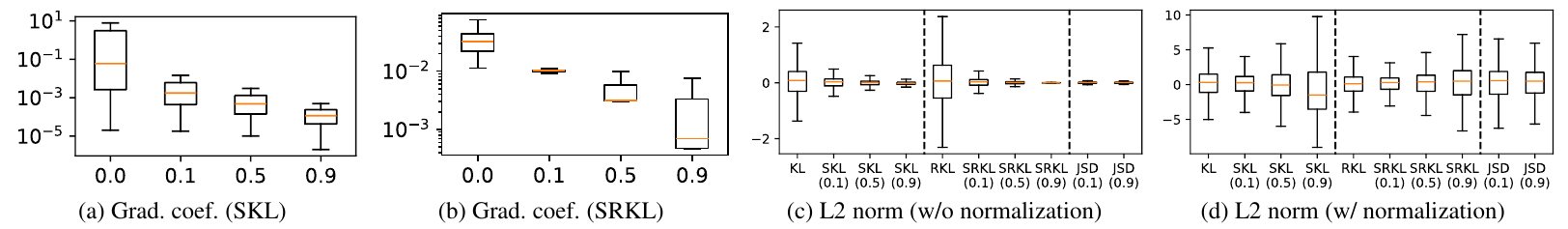}
    \vspace{-20pt}
    \caption{\textbf{(a)}-\textbf{(b)}: Gradient coefficient distribution for SKL and SRKL across different skew values $\alpha$, as shown in Eq.\,\ref{eq:grad_skl}--\ref{eq:grad_srkl}. \textbf{(c)}: Distribution of differences between divergence values and their (exponential) moving average of $\alpha$-S(R)KL, as shown in Thm.\,\ref{method:thm}, and those of $\beta$-JSD by substituting SKL into JSD across different $\alpha$ and $\beta$, respectively. \textbf{(d)}: Normalized L2 norm distribution, dividing the L2 norm in \textbf{(c)} by corresponding gradient coefficient values.}
    \label{fig:overview_skl}
    \vspace{-5pt}
\end{figure*}


\section{\alg}
In this section, we present the technical details of \alg, addressing the limitations of previous methods. Our proposed method includes: (1) \textbf{\textit{Skew KLD}} (Sec.\,\ref{sec:skew_kl}), which significantly improves optimization stability and generalizability. The skew KLD loss addresses the pitfalls of previous objective functions that may lead the student model to sub-optimal, lacking analytical grounding. (2) \textbf{\textit{Adaptive off-policy approach}}~(Sec.\,\ref{sec:adaptive}), which comprises a novel adaptive SGO scheduler to balance the trade-off between noisy feedback and training-inference mismatch by minimally utilizing SGO, and off-policy strategy to improve the sample efficiency of SGO with maintaining the performance. 
We present the overall pipeline of \alg in Algorithm\,\ref{alg:aesop}.

\subsection{Skew (Reverse) KLD}\label{sec:skew_kl}
We mathematically present our motivation that skewing such KLD is highly effective in improving the performance of student models with a more favorable optimization process. The definition of skew KLD\,(SKL, \citealt{lee2001effectiveness}) employs the parameter $\alpha$ that controls the mixing ratio of two distributions. The $\alpha$-SKL between $p$ and $q_{\theta}$ is defined as the KLD between $p$ and the mixture of distributions $\alpha p + (1-\alpha) q_{\theta}$:
\begin{equation*}
    D_{\text{SKL}}^{(\alpha)}(p, q_{\theta}) = D_{\text{KL}} \left(p, \alpha p + (1-\alpha) q_{\theta}\right).
\end{equation*}
We similarly define the $\alpha$-SRKL by $D_{\text{SRKL}}^{(\alpha)}(p, q_{\theta}) = D_{\text{KL}}(q_\theta, (1-\alpha) p + \alpha q_{\theta})$.
Here, following our thorough analysis, we present a comprehensive insight suggesting that S(R)KL is superior to other loss functions, owing to its more stable gradient and smaller approximation error.

\vspace{-7.5pt}
\paragraph{Stable gradient.} 
To provide stable optimization of SKL, we first analyze the gradients of KLD and SKL to parameter $\theta$.
Given a context-target sequence pair $(\mathbf{x}, \mathbf{y})$, we define the gradient of KLD w.r.t. $\theta$ \cite{ji2023tailoring}:
\begin{equation}\label{eq:grad_dkl}
    \nabla_{\theta} D_{\text{KL}}(p,q_{\theta}) = - \mathbf{r}_{p, q_{\theta}} \nabla_{\theta} q_{\theta} (\mathbf{y}|\mathbf{x}),
\end{equation}
where $\mathbf{r}_{p_{1}, p_{2}}$ is the ratio between arbitrary distribution $p_{1}$ and $p_{2}$. The result is the model probability's negative gradient, weighted inversely by its value. If $q_{\theta}(\mathbf{y}|\mathbf{x}) \approx 0$, the gradient norm grows large, causing a significant, potentially noisy parameter update step. These ingredients can adversely affect gradient updates, impacting the optimization process.
\begin{figure*}[t]
    \centering
    \includegraphics[width=1.0\linewidth]{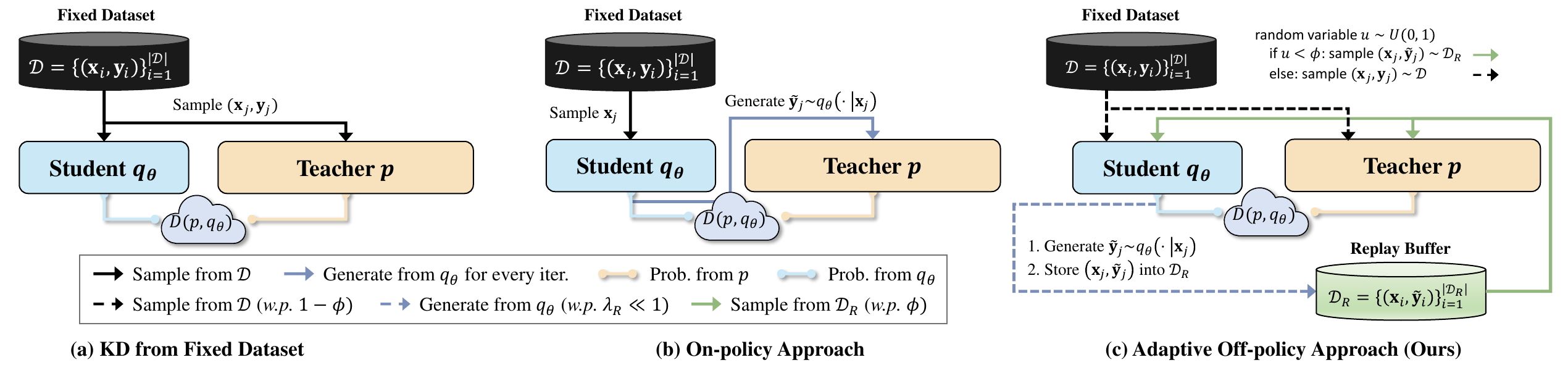}
    \vspace{-17.5pt}
    \caption{(a) KD from fixed dataset~\cite{Hinton2015DistillingTK} shows higher efficiency but lower performance. 
    (b) On-policy approach~\cite{agarwal2023gkd, gu2023knowledge} shows higher performance but lower efficiency. 
    (c) Our adaptive off-policy approach shows both higher performance and efficiency. This advantage is attributed to introducing a replay buffer and progressively decreasing a replay ratio $\zeta := (1-\frac{t}{T})$, which consequently maintains small SGO generation frequency $\lambda_{R} := \phi(1-\frac{t}{T})$ during the entire training phase.
    }
    \label{fig:overview}
    \vspace{-5pt}
\end{figure*}
We now compute the gradient of SKL w.r.t. $\theta$:
\begin{equation}\label{eq:grad_skl}
    \nabla_{\theta} D_{\text{SKL}}^{(\alpha)}(p,q_{\theta}) = - \underbrace{(1-\alpha) \mathbf{r}_{p, \tilde{q}_{\theta}}}_{\text{coefficient}} \nabla_{\theta} q_{\theta}(\mathbf{y}|\mathbf{x}),
\vspace{-5pt}
\end{equation}
where $\tilde{q}_{\theta}(\mathbf{y}|\mathbf{x}) = \alpha p(\mathbf{y}|\mathbf{x}) + (1-\alpha) q_{\theta}(\mathbf{y}|\mathbf{x})$. SKL offers a reduced gradient norm compared to KLD, due to $p$ and $q_{\theta}$ interpolation preventing the denominator of $\mathbf{r}_{p, \tilde{q}_{\theta}}$ from reaching zero. \textbf{This results in a more stable gradient for SKL.} The gradient analysis for RKLD and SRKL reveals similar trends.
\begin{align}\label{eq:grad_srkl}
    \nabla_{\theta} D_{\text{KL}}(q_{\theta}, p) &= - \left(\log \mathbf{r}_{q_{\theta}, p} + 1 \right) \nabla_{\theta} q_{\theta} (\mathbf{y}|\mathbf{x}), \\
    \nabla_{\theta} D_{\text{SKL}}^{(\alpha)}(q_{\theta},p) &= -\underbrace{\left( \log \mathbf{r}_{q_{\theta}, \tilde{p}} + 1 - \alpha \mathbf{r}_{q_{\theta}, \tilde{p}} \right)}_{\text{coefficient}} \nabla_{\theta} q_{\theta} (\mathbf{y}|\mathbf{x}), \nonumber
\vspace{-7pt}
\end{align}
where $\tilde{p}(\mathbf{y}|\mathbf{x}) = (1-\alpha) p(\mathbf{y}|\mathbf{x}) + \alpha q_{\theta}(\mathbf{y}|\mathbf{x})$. All derivations for the gradients are in Appendix\,\ref{app:gradient}.
We visualize the gradient coefficient distribution in Fig.\,\ref{fig:overview_skl}(a) and Fig.\,\ref{fig:overview_skl}(b) which verify our analysis for gradient. As $\alpha$ becomes large, the coefficient is effectively small for both SKL and SRKL.

\vspace{-7.5pt}
\paragraph{Small approximation error.} 
We show that the empirical estimator of SKL from mini-batch training has a bounded L2 norm. This bounded norm ensures that rapid convergence, with minimal error between the estimator and true divergence, yields high generalizability by accurately reflecting the full distribution from the empirical estimator.
\begin{theorem}\label{method:thm}
    Let $p^{1}_{n}$ and $p^{2}_{n}$ be empirical distributions of $n$ i.i.d. samples from $p^{1}$ and $p^{2}$, respectively. 
    Under mild assumptions, we have an upper bound for the L2 norm of $\alpha$-SKL estimator $D_{\text{SKL}}^{(\alpha)}(p^{1}_n, p^{2}_n)$ for $D_{\text{SKL}}^{(\alpha)}(p^{1}, p^{2})$:
    \begin{align*}
        \mathbb{E}[|D_{\text{SKL}}^{(\alpha)}&(p^{1}_n, p^{2}_n) - D_{\text{SKL}}^{(\alpha)}(p^{1}, p^{2})|^{2}] \\
        &\leq \frac{c_{1}(\alpha)}{n^{2}} + \frac{c_{2} \log^{2}(\alpha n)}{n} + \frac{c_3 \log^{2}(c_4 n)}{\alpha^{2}n},
    \end{align*}
    for $c_{1}(\alpha)=\min\left\{\frac{1}{\alpha^{2}}, \frac{\chi^{2}(p^{1}, p^{2})^{2}}{(1-\alpha)^{2}}\right\}$ and positive constants $c_{2}, c_{3}, c_{4}$ that are independent of $n$, $\alpha$, and $D_{\text{KL}}(p^{1}, p^{2})$, where $\chi^{2}(p^{1}, p^{2})$ is chi-square divergence between $p^{1}$, $p^{2}$.
\end{theorem} 
\vspace{-7.5pt}
The proof is in Appendix\,\ref{app:theory}. 
Thm.\,\ref{method:thm} states that a large $\alpha$ lowers the L2 norm between empirical and true objectives. 
We show in Fig.\,\ref{fig:overview_skl}(c) that $\alpha$-S(R)KL reduces the error between the value for each mini-batch and their moving average more effectively than (R)KLD. 
However, considering the gradient scale reduction in Eq.\,\ref{eq:grad_skl}--\ref{eq:grad_srkl}, the benefit of a reduced L2 norm from Thm.\ref{method:thm} is negated by compensating the reduced gradient scale of modern optimizers\,\cite{loshchilov2017decoupled}.
We further provide a statement considering the reverse of approximated gradient coefficient, $\frac{1}{(1-\alpha)}$, especially when $\mathbf{r}_{\cdot, \cdot}$ averages near 1:
\begin{remark}
    By considering the reverse of approximated gradient scale, we have:
    \begin{align*}
        \textstyle \mathbb{E}[|\frac{1}{(1-\alpha)}&(D_{\text{SKL}}^{(\alpha)}(p^{1}_n, p^{2}_n) - D_{\text{SKL}}^{(\alpha)}(p^{1}, p^{2}))|^{2}] \\
        &\leq \frac{c^{*}_{1}(\alpha)}{n^{2}} + \frac{c_{2} \log^{2}(\alpha n)}{(1-\alpha)^{2}n} + \frac{c_3 \log^{2}(c_4 n)}{\alpha^{2}(1-\alpha)^{2} n},
    \end{align*}
    for $c_{1}^{*}(\alpha) = \min \left\{\frac{1}{\alpha^{2}(1-\alpha)^{2}}, \frac{\chi^{2}(p^{1}, p^{2})^{2}}{(1-\alpha)^{4}}\right\}$.
\end{remark}
\vspace{-3pt}
Overall, selecting $\alpha$ involves a trade-off: the relationship between the upper bound of the normalized L2 norm and $\alpha \in [0,1]$ appears to be convex, underscoring the importance of balancing gradient and L2 norm scales, as shown in Fig.\,\ref{fig:overview_skl}(d). \textit{\textbf{From these results, we also discern a fundamental difference between S(R)KL and JSD: S(R)KL with a mild $\alpha$ achieves an appropriate L2 norm value, whereas $D_{\text{JSD}}^{(\beta)}(p, q_\theta) := \beta D_{\text{SKL}}^{(\beta)}(p, q_\theta) + (1-\beta) D_{\text{SKL}}^{(1-\beta)}(q_\theta, p)$ cannot simultaneously moderate skew values for both terms.}} Our experiments indicate that $\alpha$-SKL and $\alpha$-SRKL are most effective with $\alpha=0.1$, surpassing KLD, RKLD, and JSD in performance, as demonstrated in Tab.\,\ref{tab:skew} and Fig.\,\ref{fig:fig_ablation_alpha}.

\subsection{Adaptive Off-policy Approach}\label{sec:adaptive}

In Sec.\,\ref{sec:revisit}, we have discussed two main issues of na\"ively using SGO: (1) \textit{the risk of noisy feedback due to the teacher model's unfamiliarity with the SGO} and (2) \textit{the significant increase in training time}.
For instance, employing SGO at every training iteration (\textit{on-policy}, \citealt{agarwal2023gkd}) can lead to a substantial increase in runtime, up to 5.5$\times$ (refer to Fig.\,\ref{fig:relative_train_time}), and may also result in performance degradation (refer to Fig.\,\ref{fig:instruction_short}). To tackle these issues, we propose (1) an adaptive SGO scheduler to conservatively utilize SGO in KD, guided by the validation loss of student models, thus mitigating the risk of noisy feedback; and (2) an efficient off-policy strategy to improve the sample efficiency of SGO.

\vspace{-7.5pt}
\paragraph{Adaptive SGO scheduler.}
We define the probability of using SGOs, denoted as $\phi$. We apply SGOs with a probability of $\phi$, \textit{i.e.}, using samples from a fixed dataset with a probability of $1-\phi$ (refer to Fig.\,\ref{fig:overview}(a)). 
Unlike previous methods that maintain a consistently high $\phi$\,\cite{lin-etal-2020-autoregressive, agarwal2023gkd}, our scheduler starts with low $\phi$ value, gradually increasing during training. 
This strategy prevents student models from being overwhelmed by noisy feedback\,(as shown in Fig.\,\ref{fig:ill_mismatch}). 
To manage the increase of $\phi$, we primarily rely on validation loss as a metric.
Our observations indicate that training on a diverse range of SGOs, rather than solely on a fixed dataset, mitigates training-inference mismatch and consequently lowers validation loss.
We adjust $\phi$ by comparing the current and previous validation losses; an increase in validation loss leads to an increase in $\phi$. 
This method effectively improves student model performance by striking a balance between managing noisy feedback and minimizing training-inference mismatch issues. 
For further details, please refer to Appendix\,\ref{app:algorithm} and Algorithm\,\ref{alg:aesop}.

\begin{figure*}[t]
    \centering
    \vspace{-3pt}
    \includegraphics[width=1.0\linewidth]{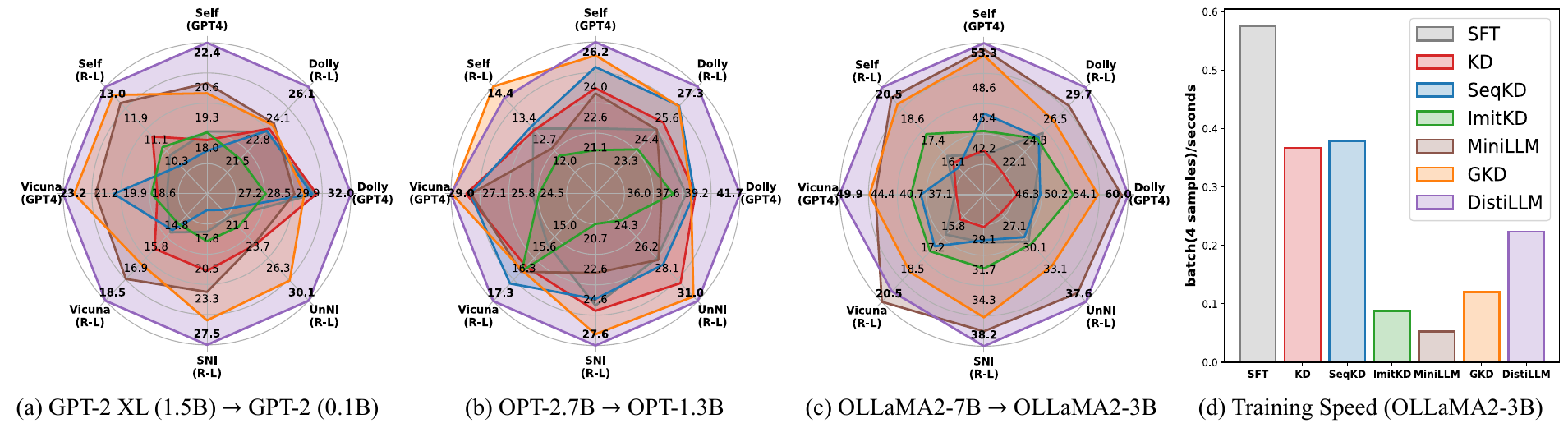}
    \vspace{-20pt}
    \caption{Instruction-following tasks, distilling GPT-2\,\cite{radford2019language}, OPT\,\cite{zhang2022opt}, and OpenLLaMA~(OLLaMA; \citealt{openlm2023openllama}) model families on \texttt{databricks-dolly-15k}. GPT4 and R-L indicate GPT-4 feedback and ROUGE-L, respectively. To evaluate training speed, results are obtained using four A100 GPUs. Further details and results are in Fig.\,\ref{fig:instruction}.}
    \label{fig:instruction_short}
\end{figure*}
\vspace{-7.5pt}
\paragraph{Off-policy approach for sample efficiency.}
To enhance efficiency, we replace the recently adopted on-policy approach\,\cite{gu2023knowledge, agarwal2023gkd} with an off-policy approach, employing a replay buffer\,\cite{mnih2015human, fedus2020revisiting}, \textbf{as illustrated in Fig.\,\ref{fig:overview}(c)}. 
In this buffer, we store SGOs from student models at a probability of $\lambda_{R}$, as indicated by \textbf{the frequency of the dashed blue arrows}.
Subsequently, we randomly draw samples from this pool. 
We also replace the oldest samples of $\mathcal{D}_{R}$ with new ones once it reaches its maximum capacity. 
This off-policy strategy significantly improves the sample efficiency of KD with SGO, saving more resources than on-policy which constantly requires new data.

Off-policy reinforcement learning is prone to high bias error\,\cite{kumar2019stabilizing, lee2023plastic}, particularly when there is a significant divergence between past and current policies, using samples from the past policy becomes suboptimal for updating the current policy.
To address this, we set $\lambda_{R} := \phi (1 - \frac{t}{T})$, where $t$ represents the current training iteration and $T$ is the total number of iterations.
To explain the philosophy of our design, we define $\zeta := (1 - \frac{t}{T})$ as the replay ratio:
\begin{itemize}[leftmargin=*, itemsep=0pt]
\vspace{-7.5pt}
    \item In the \textbf{\textit{early training phase}} (\textit{i.e.}, when adaptive probability $\phi$ is small), where student model parameters rapidly evolve, we focus on using current SGOs with a high replay ratio to minimize bias error.
    \item In the \textbf{\textit{late stages of training}} (\textit{i.e.}, when $\phi$ is larger), as the student model nears convergence, we predominantly reuse stored SGOs from $\mathcal{D}_{R}$ with a small replay ratio.
    \vspace{-7.5pt}
\end{itemize}
Hence, we can consistently maintain a small $\lambda_{R}$ by applying a large $\zeta$ for relatively smaller $\phi$ values and vice versa, effectively balancing bias error reduction with sample efficiency. Our design of $\zeta$ shows higher efficiency and comparable performance with its alternatives as shown in Tab.\,\ref{tab:replay}.

\vspace{-7.5pt}
\paragraph{Synergy with SKL.} Off-policy approach's success stems from the fast convergence speed of S(R)KL while other loss functions cannot be achieved. As Fig.\,\ref{fig:skew} shows, both SKL and SRKL have a significant early-stage improvement, effectively leveraging the off-policy approach without high bias issues. This advantage is also evident in Tab.\,\ref{tab:off}, where, unlike other baselines\,\cite{lin-etal-2020-autoregressive, agarwal2023gkd} that suffer performance drops when switching from on-policy to off-policy, our method maintains its efficacy. As a result, we verify that our off-policy approach significantly improves the training efficiency with a negligible performance drop as depicted in Tab.\,\ref{tab:genfilt} and Fig.\,\ref{fig:relative_train_time}.
\section{Experiments}
We evaluate \alg on instruction-following, text summarization, and machine translation tasks. We apply \alg with SRKL and the off-policy approach with initial probability as zero and replay buffer size of 1000 as SRKL with $\alpha$ of 0.1 as determined through our ablation studies in Sec.\,\ref{sec:exp}. We compare our approach with previous KD: (1)\,supervised fine-tuning (SFT) directly fine-tunes the student on fixed datasets; (2) KD\,\cite{Hinton2015DistillingTK} uses KLD on fixed datasets; (3) SeqKD\,\cite{kim-rush-2016-sequence} applies SFT to teacher-generated output; (4) ImitKD\,\cite{lin-etal-2020-autoregressive} employs KLD on SGO; (5) MiniLLM\,\cite{gu2023knowledge} utilizes a policy gradient approach on SGO; and (6) GKD\,\cite{agarwal2023gkd} uses JSD on a mixture of SGOs and a fixed dataset. Further details on the experimental setup are found in Appendix\,\ref{app:setup}.

\begin{figure*}[t]
\vspace{-5pt}
\centering
\small
\begin{minipage}[b]{0.72\textwidth}
\centering
\captionof{table}{Evaluation of the effect of SKL and SRKL loss functions. \textbf{Bold} and \underline{underline} indicate the best and second-best results, respectively, among those from the same evaluation dataset. We report the average and standard deviation of ROUGE-L scores across five random seeds. Green~(\greenup) and red~(\reddown) arrows indicate improvement and deterioration over the corresponding baselines.}
\vspace{5pt}
\resizebox{1.0\textwidth}{!}{
\addtolength{\tabcolsep}{1.5pt}
\begin{tabular}{l|c|c|c|c|c}
\toprule[0.1em]
        Loss Function & \multicolumn{1}{c|}{Dolly Eval\,($\uparrow$)} & \multicolumn{1}{c|}{Self-Instruct\,($\uparrow$)} & \multicolumn{1}{c|}{Vicuna Eval\,($\uparrow$)} & \multicolumn{1}{c|}{Super-Natural\,($\uparrow$)} & Unnatural\,($\uparrow$) \\ \midrule
        KLD & 23.52 (0.22) & 11.23 (0.46) & 15.92 (0.41) & 20.68 (0.16) & 23.38 (0.13) \\
        RKLD & 23.82 (0.34) & 10.90 (0.58) & \underline{16.11 (0.46)} & 22.47 (0.21) & 23.03 (0.11) \\
        Generalized JSD & 24.34 (0.35) & 12.01 (0.54) & 15.21 (0.61) & 25.08 (0.36) & 27.54 (0.07) \\ \midrule
        SKL & \underline{24.80 (0.12)} \greenup & \underline{12.86 (0.34)} \greenup & \textbf{16.20 (0.57)} \greenup & \textbf{26.26 (0.41)} \greenup & \underline{28.06 (0.08)} \greenup \\
        SRKL & \textbf{25.21 (0.27)} \greenup & \textbf{12.98 (0.24)} \greenup & 15.77 (0.39) \reddown & \underline{25.83 (0.15)} \greenup & \textbf{28.62 (0.10)} \greenup \\
\bottomrule[0.1em]
\end{tabular}
}\label{tab:skew}
\end{minipage}
\hfill
\begin{minipage}[t]{0.26\textwidth}
\centering
\small
\includegraphics[width=\textwidth]{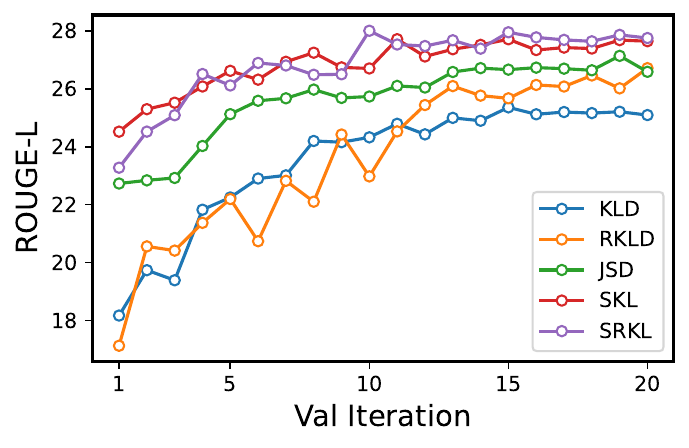}
\vspace{-22pt}
\caption{ROUGE-L scores for the validation set across the different loss functions.}
\label{fig:skew}
\end{minipage}
\end{figure*}

\begin{figure*}[t]
\centering
\small
\begin{minipage}[b]{0.73\textwidth}
\centering
\captionof{table}{Evaluation of the adaptive off-policy approach. We report the average and standard deviation of ROUGE-L across five random seeds. The best and second best performances are highlighted \textbf{bold} and \underline{underline}. Green\,(\greenup) and red\,(\reddown) arrows indicate improvement and deduction over the baselines.} 
\vspace{5pt}
\resizebox{1.0\textwidth}{!}{
\addtolength{\tabcolsep}{0.5pt}
\begin{tabular}{l|c|c|c|c|c}
\toprule[0.1em]
        Generation & \multicolumn{1}{c|}{Dolly Eval\,($\uparrow$)} & \multicolumn{1}{c|}{Self-Instruct\,($\uparrow$)} & \multicolumn{1}{c|}{Vicuna Eval\,($\uparrow$)} & \multicolumn{1}{c|}{Super-Natural\,($\uparrow$)} & Unnatural\,($\uparrow$) \\ \midrule[0.1em]
        Skew KLD & 24.80 (0.12) & 12.86 (0.34) & 16.20 (0.57) & 26.26 (0.41) & 28.06 (0.08) \\ \midrule
        $\llcorner$ On-policy & 24.27 (0.46) \reddown & \underline{13.13 (0.44)} \greenup & 16.39 (0.21) \greenup & 25.87 (0.18) \reddown & 26.49 (0.09) \reddown \\
        $\llcorner$ Mixed & 25.27 (0.35) \greenup & 12.24 (0.69) \reddown & 17.19 (0.29) \greenup & 25.30 (0.33) \reddown & 26.51 (0.11) \reddown \\
        $\llcorner$ Adaptive (ours) & \textbf{25.90 (0.20)} \greenup & \textbf{13.24 (0.30)} \greenup & \textbf{17.59 (0.44)} \greenup & \textbf{27.62 (0.05)} \greenup & \textbf{28.30 (0.11)} \greenup \\
        \;\;+ Off-policy (ours) & \underline{25.79 (0.28)} \greenup & 13.03 (0.29) \greenup & \underline{17.41 (0.15)} \greenup & \underline{27.32 (0.09)} \greenup & \underline{28.13 (0.21)} \greenup \\ \midrule[0.1em]
        Skew RKLD & 25.21 (0.27) & 12.98 (0.24) & 15.77 (0.39) & 25.83 (0.15) & 28.62 (0.10) \\ \midrule
        $\llcorner$ On-policy & 26.04 (0.33) \greenup & 12.93 (0.54) \reddown & 17.45 (0.37) \greenup & 27.29 (0.12) \greenup & 28.72 (0.10) \greenup \\
        $\llcorner$ Mixed & 26.01 (0.61) \greenup & 12.24 (0.69) \reddown & 17.19 (0.29) \greenup & 26.40 (0.34) \greenup & 29.02 (0.14) \greenup \\
        $\llcorner$ Adaptive (ours) & \textbf{26.37 (0.21)} \greenup & \underline{13.14 (0.37)} \greenup & \underline{18.32 (0.17)} \greenup & \textbf{28.24 (0.22)} \greenup & \textbf{30.11 (0.04)} \greenup \\
        \;\;+ Off-policy (ours) & \underline{26.11 (0.68)} \greenup & \textbf{13.14 (0.69)} \greenup & \textbf{18.46 (0.53)} \greenup & \underline{27.51 (0.03)} \greenup & \underline{29.35 (0.07)} \greenup \\
\bottomrule[0.1em]
\end{tabular}
}\label{tab:genfilt}
\end{minipage}
\hfill 
\begin{minipage}[t]{0.25\textwidth}
\centering
\small
\includegraphics[width=\textwidth]{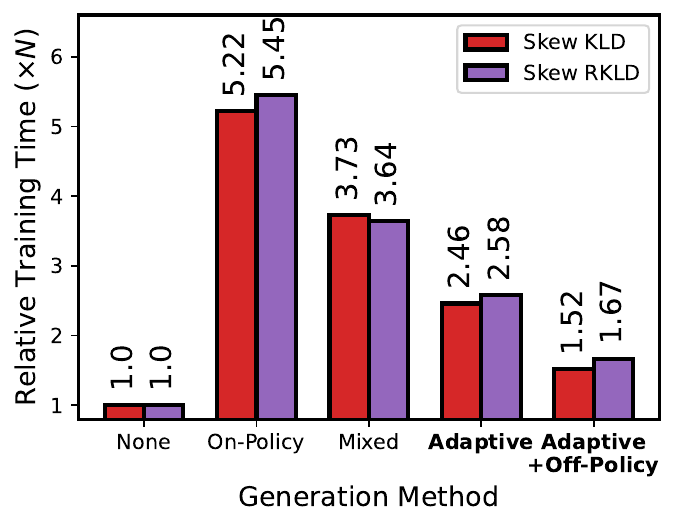}
\vspace{-22pt}
\caption{Relative training time for different generation methods for S(R)KL. The adaptive off-policy approach shows significant efficiency.}
\label{fig:relative_train_time}
\end{minipage}
\end{figure*}
\subsection{Task-Agnostic Instruction-Following}\label{comp:inst}
\paragraph{Implementation details.}
Our implementation follows the experiment setup of \citet{gu2023knowledge}. We first construct the training data from \texttt{databricks-dolly-15k}\,\cite{DatabricksBlog2023DollyV2}, wherein we randomly select 14K samples for training and equally leave 500 samples for validation and testing, respectively. We evaluate the trained models on five instruction-following benchmarks: Dolly evaluation, Self-Instruct\,\cite{wang-etal-2023-self-instruct}, Vicuna evaluation, Super-Natural Instructions\,\cite{wang-etal-2022-super}, and Unnatural Instruction\,\cite{honovich-etal-2023-unnatural}. 
Similar to \citet{ouyang2022training} and \citet{gu2023knowledge}, we add a language modeling\,\cite{radford2018improving} loss to the OpenWebText\,\cite{Gokaslan2019OpenWeb} corpus for all experiments.
Employing such an additional loss function on the pre-training corpus has been shown to effectively improve instruction-tuning performance, as demonstrated in the work of \citet{gu2023knowledge}.
For evaluation, we adopt two metrics: ROUGE-L\,\cite{lin-2004-rouge} and GPT-4 feedback\,\cite{zheng2023judging}. 

\vspace{-7.5pt}
\paragraph{Main results.}
Fig.\,\ref{fig:instruction_short} illustrates the instruction-following performances, demonstrating \alg's superiority to state-of-the-art methods, across diverse teacher-student combinations and evaluation metrics.
Notably, GKD and MiniLLM, which recently incorporated SGO in distillation, are less effective or even lead to a performance decline in smaller-sized student models, particularly in GPT-4 feedback and some ROUGE-L scores.
Further details and results can be found in Appendix\,\ref{app:instruction} and Tab.\,\ref{tab:instruction_gpt}--\ref{tab:instruction_llama}. 

For evaluating \alg on larger-sized LLMs, we utilize low-rank adaptation~(LoRA; \citealt{hu2022lora}) for training. Specifically, we employ OpenLLaMA2-7B\,\cite{openlm2023openllama} as the teacher model and OpenLLaMA2-3B as the student model. The results, as depicted in Fig.\,\ref{fig:instruction_short}(c), reveal that \alg significantly surpasses other baseline methods in performance. Notably, while other supervised KD techniques are less effective than MiniLLM in LLM applications, \alg uniquely achieves superior performance compared to MiniLLM. This outcome, particularly on task-agnostic instruction-following datasets, underscores \alg's effectiveness in general-purpose LLMs. Moreover, as we can observe in Fig.\,\ref{fig:instruction_short}(d), \alg requires only 1.6$\times$ the training time of na\"ive KD, whereas other methods take 3\,$\sim$\,7$\times$. This demonstrates the efficiency of the proposed \alg and its broad applicability to expensive LLMs.
\vspace{-5pt}


\begin{table}[t]
\centering
\small
\vspace{-7.5pt}
\caption{
Text summarization and machine translation results, evaluated with ROUGE-L and BLEU scores, respectively. Using the T5-XL\,(3B) teacher model, we fine-tune the student models, T5-Base\,(0.2B) and T5-Small\,(0.06B).
}
\vspace{5pt}
\resizebox{1.0\columnwidth}{!}{
\addtolength{\tabcolsep}{-2.5pt}
\begin{tabular}{l|cc|cc}
\toprule[0.1em]
        Dataset & \multicolumn{2}{c|}{\!SAMSum\!} & \multicolumn{2}{c}{IWSLT 2017 En-De} \\ \midrule
        \textbf{\textit{T5-XL $\rightarrow$}} & \textbf{\textit{T5-Base}} & \textbf{\textit{T5-Small}} & \textbf{\textit{T5-Base}} & \textbf{\textit{T5-Small}} \\ \midrule
        KD~\cite{Hinton2015DistillingTK} & 46.23 & 39.52 & 29.36 & 21.15 \\
        SeqKD~\cite{kim-rush-2016-sequence} & 46.89 & 40.24 & 29.07 & 21.42 \\
        ImitKD~\cite{lin-etal-2020-autoregressive} & 48.57 & 41.44 & 29.87 & 21.52 \\
        GKD~\cite{agarwal2023gkd} & 48.49 & 41.92 & 30.24 & 22.04 \\
        \alg \textbf{(ours)} & \textbf{49.11} & \textbf{42.37} & \textbf{30.32} & \textbf{22.53} \\
\bottomrule[0.1em]
\end{tabular}
}\label{tab:summarization}
\vspace{-10pt}
\end{table}

\subsection{Text Summarization and Machine Translation}\label{comp:summ}
We evaluate the effectiveness of task-specific LMs on summarization and translation tasks using two datasets, SAMSum\,\cite{gliwa2019samsum} and IWSLT 2017\,\cite{cettolo-etal-2017-overview}. For the SAMSum dataset, we use T5-XL v1.1 \cite{raffel2020exploring} as the teacher model and T5-Base/-Small v1.1 as the student models. For the IWSLT dataset, we employed mT5-XL\,\cite{xue-etal-2021-mt5} as the teacher model and mT5-Base/-Small v1.1 as the student models.

Tab.\,\ref{tab:summarization} displays ROUGE-L and BLEU\,\cite{papineni-etal-2002-bleu} scores for student models. We observe that \alg outperforms other baselines, but has a smaller performance margin in single-task scenarios than in general-purpose instruction-following tasks.
In the SAMSum, students trained with ImitKD outperform those trained with GKD, while in the IWSLT, GKD outperforms ImitKD. These findings align with the results in \citet{agarwal2023gkd}, which indicate that the effectiveness of objective functions and the use of SGO are task-dependent.
Despite these variations, \alg consistently shows superiority across different tasks, attributed to its adaptive use of the SGO scheduler and skew divergence.
Additional details on the performance in the XSum\,\cite{narayan2018don} and CNN/DM\,\cite{see-etal-2017-get} datasets are provided in Appendix\,\ref{app:summarization}, further emphasizing the superior performance of \alg.


\vspace{-5pt}
\section{Analysis and Discussion}\label{sec:exp}
We conduct experimental analyses to verify the effectiveness of each component of \alg, distilling GPT-2 XL\,$\rightarrow$\,GPT-2 in instruction-following datasets.

\vspace{-5pt}
\subsection{Effect of Skew Divergence}\label{sec:skew_emp}
In Tab.~\ref{tab:skew}, we compare the performance of various models trained with different objective functions: conventional KLD, RKLD, and JSD with a $\beta$ of 0.9~\cite{agarwal2023gkd}, as well as our SKL and SRKL with a $\alpha$ of 0.1. The results show that our proposed objective functions generally outperform the others.
Notably, as Fig.~\ref{fig:skew} illustrates, both SKL and SRKL achieve remarkably high validation ROUGE-L scores during the entire training phase, consistently demonstrating the rapid convergence and strong generalization capabilities of our proposed loss functions. These empirical results verify our theoretical analysis in Sec.\,\ref{sec:skew_kl} and indicate that even our simple modification leads to significant performance enhancements.

\begin{figure}[t]
    \centering
    \vspace{-2pt}
    \includegraphics[width=1.0\linewidth]{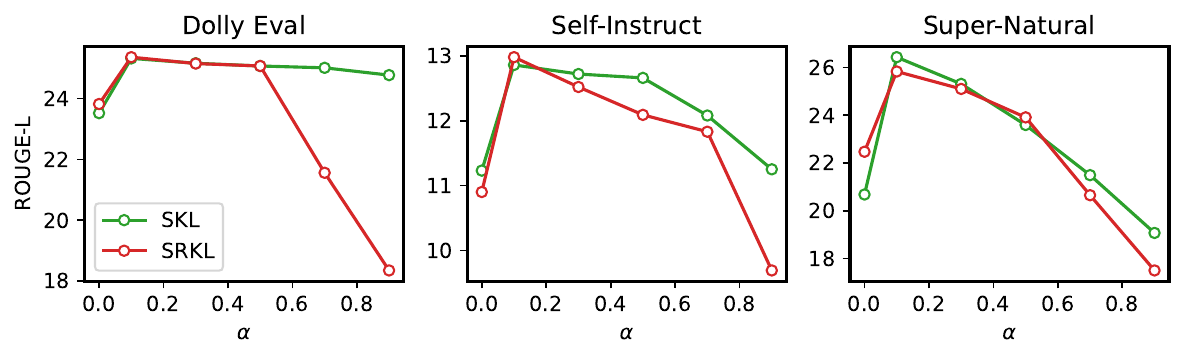}
    \vspace{-25pt}
    \caption{Comparison of the ROUGE-L score using different $\alpha$.}
    \label{fig:fig_ablation_alpha}
    \vspace{-10pt}
\end{figure}
\subsection{Effect of Adaptive Off-policy Approach}\label{exp:aesop}
In Tab.\,\ref{tab:genfilt} and Fig.~\ref{fig:relative_train_time}, we confirm the effectiveness and efficiency of our adaptive off-policy approach by comparing it with an on-policy\,\cite{lin-etal-2020-autoregressive} and a mixed strategy\,\cite{agarwal2023gkd}.
In the mixed strategy, we use the on-policy approach with a probability of 0.5; otherwise, we sample from the fixed dataset, following the original work.
The results indicate that our adaptive SGO scheduler effectively balances the trade-off between the risk of noisy feedback and training-inference mismatch.
Notably, while the baselines suffer from performance degradation when applying SKL, our proposed adaptive strategy consistently improves the performance for all datasets.
Moreover, the off-policy approach leads to a minimal performance drop while significantly improving computational efficiency. It achieves 2.2$\times$ to 3.4$\times$ faster training speed compared to the on-policy or mixed strategy.

\begin{table}[t]
\centering
\vspace{-5pt}
\caption{Application of our off-policy method to the existing KD methods. Off-policy significantly reduces the performance of ImitKD and GKD, as opposed to our proposed \alg.}
\vspace{5pt}
\resizebox{1.0\columnwidth}{!}{
\addtolength{\tabcolsep}{-2pt}
\begin{tabular}{l|cc|cc|cc}
\toprule[0.1em]
        Dataset & \multicolumn{2}{c|}{Dolly Eval} & \multicolumn{2}{c|}{Self-Instruct} & \multicolumn{2}{c}{Super-Natural} \\ \midrule
        Sampling & on- & off- & on- & off- & on- & off- \\ \midrule
        ImitKD~\cite{lin-etal-2020-autoregressive} & 21.63 & 20.62 & 10.85 & 10.09 & 19.94 & 18.04 \\ 
        GKD~\cite{agarwal2023gkd} & 23.75 & 22.89 & 12.73 & 12.78 & 26.05 & 24.97 \\
        \alg \textbf{(ours)} & 26.37 & 26.12 & 13.14 & 13.16 & 28.24 & 28.20 \\
\bottomrule[0.1em]
\end{tabular}
}
\label{tab:off}
\end{table}
\begin{table}[t]
\centering
\vspace{-10pt}
\caption{
Comparison of the performance from the adaptive SGO scheduler and its terminal probability~(0.4 for SKL and 0.3 for SRKL) and the best performance from manually tuned probability and corresponding value~(\textit{i.e.}, results in parenthesis).
}
\vspace{5pt}
\resizebox{1.0\columnwidth}{!}{
\addtolength{\tabcolsep}{-1.3pt}
\begin{tabular}{l|c|ccccc}
\toprule[0.1em]
        Loss & Method & Dolly & Self-Inst & Vicuna & SNI & UNI \\ \midrule
        \multirow{2}{*}{SKL} & Adapt. & 25.90 & 13.24 & 17.59 & 27.62 & 28.30 \\
        & Best & 25.15\,(0.3) & 13.17\,(0.5) & 17.04\,(0.3) & 27.18\,(0.4) & 28.33\,(0.6) \\ \midrule
        \multirow{2}{*}{SRKL} & Adapt. & 26.37 & 13.14 & 18.32 & 28.24 & 30.11 \\
        & Best & 26.38\,(0.4) & 12.98\,(0.3) & 17.88\,(0.3) & 28.24\,(0.3) & 30.02\,(0.3) \\ 
\bottomrule[0.1em]
\end{tabular}
}
\label{tab:prob}
\vspace{-10pt}
\end{table}
\subsection{Additional Ablation Studies on \alg}\label{exp:ablation}
\paragraph{Skew values $\alpha$.}
As we highlight the importance of proper choice $\alpha$ in Sec.~\ref{sec:skew_kl}, we empirically evaluate the performance of $\alpha$-SKL under the wide range of $\alpha$. 
Fig.~\ref{fig:fig_ablation_alpha} illustrates that both SKL and SRKL achieve the best performance on the $\alpha$ value of 0.1. 
These results are highly consistent with the result in Fig.~\ref{fig:overview_skl}, where both normalized L2 norms have the smallest values in $\alpha=0.1$ and prove the validity of our theoretical analyses in Sec.~\ref{sec:skew_kl}. We observed that SKL shows a mild performance reduction as $\alpha$ increases above 0.1, the performance reduction of the SRKL is comparably severe which is discussed in detail in Appendix~\ref{app:sense_alpha}.

\vspace{-7.5pt}
\paragraph{Combining off-policy with existing KD methods.} To verify the synergy between our proposed objectives and the off-policy approach, we replaced the on-policy approach in GKD and ImitKD with an off-policy method. Given that the JSD and KLD exhibit slower training speeds compared to our SRKL, as shown in Fig.~\ref{fig:skew}, we observe that the effectiveness of the off-policy approach is significantly lower than that of our proposed \alg as reported in Tab.~\ref{tab:off}. These results validate the substantial effectiveness of combining our two proposed components.

\vspace{-7.5pt}
\paragraph{Appropriateness of adaptive probability.} Tab.~\ref{tab:prob} compares the performance of probability values from the adaptive SGO scheduler with the best performance among those obtained from manually defined probabilities. Our results show that, in most cases, the performance using our method exceeds that of the manually selected probabilities. Furthermore, the probability values derived from our adaptive SGO scheduler are close to the optimal manually defined probabilities, with the differences being no more than 0.2 for SKL and 0.1 for SRKL.

\vspace{-7.5pt}
\paragraph{One-stage Distillation}
\newcommand{\cmark}{\ding{51}}%
\newcommand{\xmark}{\ding{55}}%

\begin{table}[t]
\centering
\vspace{-5pt}
\caption{Comparison of the performance of GPT-2 student across different KD methods and initialization~(with or without fine-tuning before KD).}
\vspace{5pt}
\resizebox{1.0\columnwidth}{!}{
\addtolength{\tabcolsep}{-2pt}
\begin{tabular}{l|cc|cc|cc}
\toprule[0.1em]
        Dataset & \multicolumn{2}{c|}{Dolly Eval} & \multicolumn{2}{c|}{Self-Instruct} & \multicolumn{2}{c}{Super-Natural} \\ \midrule
        Fine-tuned & \cmark & \xmark & \cmark & \xmark & \cmark & \xmark \\ \midrule
        MiniLLM~\cite{gu2023knowledge} & 23.84 & 22.56 & 12.44 & 10.47 & 22.62 & 21.15 \\ 
        GKD~\cite{agarwal2023gkd} & 23.75 & 23.15 & 12.73 & 11.34 & 26.05 & 24.48 \\
        \alg \textbf{(SKL)} & 25.79 & 25.75 & 13.03 & 12.34 & 27.32 & 27.24 \\
        \alg \textbf{(SRKL)} & 26.11 & 26.14 & 13.14 & 12.79 & 27.51 & 26.85 \\
\bottomrule[0.1em]
\end{tabular}
}
\vspace{-10pt}
\label{tab:one}
\end{table}
One significant issue with previous KD methods using SGO is their reliance on beginning with fine-tuned student models. Without fine-tuning, student models may produce degenerated SGOs, leading to substantial noisy feedback~\cite{gu2023knowledge}. A key advantage of our proposed \alg is its rapid convergence speed and innovative adaptive SGO scheduler, which eliminates the need for such fine-tuned student models. Consequently, we evaluated the performance of students distilled from pre-trained parameters without any prior fine-tuning. Tab.~\ref{tab:one} demonstrates \alg's relative robustness to the initial state of the student model (\textit{i.e.}, regardless of whether the student was fine-tuned before KD). These findings confirm that \alg can maintain efficiency with only a minor reduction in performance, a feat not matched by other methods.



\vspace{-5pt}
\section{Related Work}

KD\,\cite{Hinton2015DistillingTK} effectively compresses neural networks, allowing smaller student models to match the performance of larger teacher models.
Recently, KD has been extended to compressing auto-regressive LMs, such as GPT-3\,\cite{brown2020gpt3}, to address the challenges posed by the large scale of current LLMs\,\cite{touvron2023llama2, anil2023palm2}, making them more viable in compute-intensive frameworks.
One popular direction of KD for auto-regressive LMs is to harness LLMs as supervising data generators where only the teacher predictions are accessible like ChatGPT~\cite{openai2023gpt4} APIs.
This line of research employed LLMs for guided annotations of unlabeled data\,\cite{alpaca, peng2023instruction} or for imparting reasoning capabilities\,\cite{wang2023scott, hsieh2023distilling}, where the resulting generated data are used for fine-tuning smaller LMs.

Another noteworthy approach, when the teacher model is accessible, entails matching the student model's generation distribution with that of the teacher model through divergence loss functions.
Recent studies\,\cite{lin-etal-2020-autoregressive, wen-etal-2023-f, gu2023knowledge, agarwal2023gkd} have focused on finding the proper objectives or using datasets during the distillation for auto-regressive LMs. ImitKD\,\cite{lin-etal-2020-autoregressive} demonstrated the effectiveness of SGO in distillation, leading to \citet{agarwal2023gkd} propose on-policy approach of SGO with diverse objectives like RKLD and JSD. \citet{wen-etal-2023-f} examined various f-divergences, including total variation distance and JSD, in auto-regressive LMs, while \citet{gu2023knowledge} proposed a policy gradient-based method addressing the high variance issues in RL-based methods. Building on this research, we introduce an effective, efficient KD method, \alg, with comprehensive analysis in objective function and data utilization.

\vspace{-5pt}
\section{Conclusion}
\vspace{-2pt}
We have proposed \alg to address the challenges of KD frameworks for auto-regressive LMs. 
Our approach incorporates two key components: (1) SKL which is based on mathematically in-depth analyses and empirical evidence; 
(2) an adaptive off-policy approach that enhances the utility of SGO by reducing the noisy feedback introduced by SGO and improving the sample efficiency with a replay buffer.
Through extensive experiments on various generation tasks, we have demonstrated the superior performance of \alg, achieving significant training efficiency and performance improvement.

\section*{Impact Statement}
This paper aims to improve the efficiency and effectiveness of KD for auto-regressive LMs, such as open-source LLMs~(\textit{e.g.,} LLaMA-2~\cite{touvron2023llama2} and Falcon~\cite{falcon180b}). A notable feature of our methodology is its dual efficacy in improving performance and computational efficiency. This approach is significant, particularly as it reduces reliance on high-end computational resources for training such small-sized LLMs. We believe that this work does not present any direct ethical concerns. Therefore, a detailed ethical discussion is not considered necessary at this stage. However, we recognize the importance of ethical considerations and are open to further discussion should our work undergo an ethics review.

\section*{Acknowledgements}
This work was supported by Institute of Information \& communications Technology Planning \& Evaluation (IITP) grant funded by the Korea government (MSIT) (No.2019-0-00075, Artificial Intelligence Graduate School Program (KAIST), 10\%) and the Institute of Information \& communications Technology Planning \& Evaluation (IITP) grant funded by the Korea government (MSIT) (No. 2022-0-00871, Development of AI Autonomy and Knowledge Enhancement for AI Agent Collaboration, 90\%).

\bibliography{main}

\begin{thebibliography}{52}
\providecommand{\natexlab}[1]{#1}
\providecommand{\url}[1]{\texttt{#1}}
\expandafter\ifx\csname urlstyle\endcsname\relax
  \providecommand{\doi}[1]{doi: #1}\else
  \providecommand{\doi}{doi: \begingroup \urlstyle{rm}\Url}\fi

\bibitem[Agarwal et~al.(2024)Agarwal, Vieillard, Stanczyk, Ramos, Geist, and
  Bachem]{agarwal2023gkd}
Agarwal, R., Vieillard, N., Stanczyk, P., Ramos, S., Geist, M., and Bachem, O.
\newblock On-policy distillation of language models: Learning from
  self-generated mistakes.
\newblock In \emph{The Twelfth International Conference on Learning
  Representations}, 2024.
\newblock URL \url{https://openreview.net/forum?id=3zKtaqxLhW}.

\bibitem[Almazrouei et~al.(2023)Almazrouei, Alobeidli, Alshamsi, Cappelli,
  Cojocaru, Alhammadi, Daniele, Heslow, Launay, Malartic, Noune, Pannier, and
  Penedo]{falcon180b}
Almazrouei, E., Alobeidli, H., Alshamsi, A., Cappelli, A., Cojocaru, R.,
  Alhammadi, M., Daniele, M., Heslow, D., Launay, J., Malartic, Q., Noune, B.,
  Pannier, B., and Penedo, G.
\newblock The falcon series of language models: Towards open frontier models.
\newblock \emph{https://arxiv.org/abs/2311.16867}, 2023.

\bibitem[Anil et~al.(2023)Anil, Dai, Firat, Johnson, Lepikhin, Passos, Shakeri,
  Taropa, Bailey, Chen, et~al.]{anil2023palm2}
Anil, R., Dai, A.~M., Firat, O., Johnson, M., Lepikhin, D., Passos, A.,
  Shakeri, S., Taropa, E., Bailey, P., Chen, Z., et~al.
\newblock Palm 2 technical report.
\newblock \emph{arXiv preprint arXiv:2305.10403}, 2023.

\bibitem[Arora et~al.(2022)Arora, El~Asri, Bahuleyan, and
  Cheung]{arora-etal-2022-exposure}
Arora, K., El~Asri, L., Bahuleyan, H., and Cheung, J.
\newblock Why exposure bias matters: An imitation learning perspective of error
  accumulation in language generation.
\newblock In Muresan, S., Nakov, P., and Villavicencio, A. (eds.),
  \emph{Findings of the Association for Computational Linguistics: ACL 2022},
  pp.\  700--710, Dublin, Ireland, May 2022. Association for Computational
  Linguistics.
\newblock \doi{10.18653/v1/2022.findings-acl.58}.
\newblock URL \url{https://aclanthology.org/2022.findings-acl.58}.

\bibitem[Brown et~al.(2020)Brown, Mann, Ryder, Subbiah, Kaplan, Dhariwal,
  Neelakantan, Shyam, Sastry, Askell, et~al.]{brown2020gpt3}
Brown, T., Mann, B., Ryder, N., Subbiah, M., Kaplan, J.~D., Dhariwal, P.,
  Neelakantan, A., Shyam, P., Sastry, G., Askell, A., et~al.
\newblock Language models are few-shot learners.
\newblock \emph{Advances in neural information processing systems},
  33:\penalty0 1877--1901, 2020.

\bibitem[Cettolo et~al.(2017)Cettolo, Federico, Bentivogli, Niehues,
  St{\"u}ker, Sudoh, Yoshino, and Federmann]{cettolo-etal-2017-overview}
Cettolo, M., Federico, M., Bentivogli, L., Niehues, J., St{\"u}ker, S., Sudoh,
  K., Yoshino, K., and Federmann, C.
\newblock Overview of the {IWSLT} 2017 evaluation campaign.
\newblock In \emph{Proceedings of the 14th International Conference on Spoken
  Language Translation}, pp.\  2--14, Tokyo, Japan, December 14-15 2017.
  International Workshop on Spoken Language Translation.
\newblock URL \url{https://aclanthology.org/2017.iwslt-1.1}.

\bibitem[Chiang et~al.(2023)Chiang, Li, Lin, Sheng, Wu, Zhang, Zheng, Zhuang,
  Zhuang, Gonzalez, Stoica, and Xing]{vicuna2023}
Chiang, W.-L., Li, Z., Lin, Z., Sheng, Y., Wu, Z., Zhang, H., Zheng, L.,
  Zhuang, S., Zhuang, Y., Gonzalez, J.~E., Stoica, I., and Xing, E.~P.
\newblock Vicuna: An open-source chatbot impressing gpt-4 with 90\%* chatgpt
  quality, March 2023.
\newblock URL \url{https://lmsys.org/blog/2023-03-30-vicuna/}.

\bibitem[Conover et~al.(2023)Conover, Hayes, Mathur, Xie, Wan, Shah, Ghodsi,
  Wendell, Zaharia, and Xin]{DatabricksBlog2023DollyV2}
Conover, M., Hayes, M., Mathur, A., Xie, J., Wan, J., Shah, S., Ghodsi, A.,
  Wendell, P., Zaharia, M., and Xin, R.
\newblock Free dolly: Introducing the world's first truly open
  instruction-tuned llm, 2023.
\newblock URL
  \url{https://www.databricks.com/blog/2023/04/12/dolly-first-open-commercially-viable-\\instruction-tuned-llm}.

\bibitem[Fedus et~al.(2020)Fedus, Ramachandran, Agarwal, Bengio, Larochelle,
  Rowland, and Dabney]{fedus2020revisiting}
Fedus, W., Ramachandran, P., Agarwal, R., Bengio, Y., Larochelle, H., Rowland,
  M., and Dabney, W.
\newblock Revisiting fundamentals of experience replay.
\newblock In \emph{International Conference on Machine Learning}, pp.\
  3061--3071. PMLR, 2020.

\bibitem[Geng \& Liu(2023)Geng and Liu]{openlm2023openllama}
Geng, X. and Liu, H.
\newblock Openllama: An open reproduction of llama, May 2023.
\newblock URL \url{https://github.com/openlm-research/open_llama}.

\bibitem[Gliwa et~al.(2019)Gliwa, Mochol, Biesek, and Wawer]{gliwa2019samsum}
Gliwa, B., Mochol, I., Biesek, M., and Wawer, A.
\newblock Samsum corpus: A human-annotated dialogue dataset for abstractive
  summarization.
\newblock In \emph{Proceedings of the 2nd Workshop on New Frontiers in
  Summarization}, pp.\  70--79, 2019.

\bibitem[Gokaslan et~al.(2019)Gokaslan, Cohen, Pavlick, and
  Tellex]{Gokaslan2019OpenWeb}
Gokaslan, A., Cohen, V., Pavlick, E., and Tellex, S.
\newblock Openwebtext corpus, 2019.

\bibitem[Gu et~al.(2024)Gu, Dong, Wei, and Huang]{gu2023knowledge}
Gu, Y., Dong, L., Wei, F., and Huang, M.
\newblock Mini{LLM}: Knowledge distillation of large language models.
\newblock In \emph{The Twelfth International Conference on Learning
  Representations}, 2024.
\newblock URL \url{https://openreview.net/forum?id=5h0qf7IBZZ}.

\bibitem[Hinton et~al.(2015)Hinton, Vinyals, and Dean]{Hinton2015DistillingTK}
Hinton, G.~E., Vinyals, O., and Dean, J.
\newblock Distilling the knowledge in a neural network.
\newblock \emph{ArXiv}, abs/1503.02531, 2015.
\newblock URL \url{https://api.semanticscholar.org/CorpusID:7200347}.

\bibitem[Honovich et~al.(2023)Honovich, Scialom, Levy, and
  Schick]{honovich-etal-2023-unnatural}
Honovich, O., Scialom, T., Levy, O., and Schick, T.
\newblock Unnatural instructions: Tuning language models with (almost) no human
  labor.
\newblock In \emph{Proceedings of the 61st Annual Meeting of the Association
  for Computational Linguistics (Volume 1: Long Papers)}, pp.\  14409--14428,
  Toronto, Canada, July 2023. Association for Computational Linguistics.
\newblock \doi{10.18653/v1/2023.acl-long.806}.
\newblock URL \url{https://aclanthology.org/2023.acl-long.806}.

\bibitem[Hsieh et~al.(2023)Hsieh, Li, Yeh, Nakhost, Fujii, Ratner, Krishna,
  Lee, and Pfister]{hsieh2023distilling}
Hsieh, C.-Y., Li, C.-L., Yeh, C.-K., Nakhost, H., Fujii, Y., Ratner, A.,
  Krishna, R., Lee, C.-Y., and Pfister, T.
\newblock Distilling step-by-step! outperforming larger language models with
  less training data and smaller model sizes.
\newblock \emph{arXiv preprint arXiv:2305.02301}, 2023.

\bibitem[Hu et~al.(2022)Hu, yelong shen, Wallis, Allen-Zhu, Li, Wang, Wang, and
  Chen]{hu2022lora}
Hu, E.~J., yelong shen, Wallis, P., Allen-Zhu, Z., Li, Y., Wang, S., Wang, L.,
  and Chen, W.
\newblock Lo{RA}: Low-rank adaptation of large language models.
\newblock In \emph{International Conference on Learning Representations}, 2022.
\newblock URL \url{https://openreview.net/forum?id=nZeVKeeFYf9}.

\bibitem[Ji et~al.(2023)Ji, Ke, Hu, Zhang, and Huang]{ji2023tailoring}
Ji, H., Ke, P., Hu, Z., Zhang, R., and Huang, M.
\newblock Tailoring language generation models under total variation distance.
\newblock In \emph{The Eleventh International Conference on Learning
  Representations}, 2023.
\newblock URL \url{https://openreview.net/forum?id=VELL0PlWfc}.

\bibitem[Kim \& Rush(2016)Kim and Rush]{kim-rush-2016-sequence}
Kim, Y. and Rush, A.~M.
\newblock Sequence-level knowledge distillation.
\newblock In \emph{Proceedings of the 2016 Conference on Empirical Methods in
  Natural Language Processing}, pp.\  1317--1327, Austin, Texas, November 2016.
  Association for Computational Linguistics.
\newblock \doi{10.18653/v1/D16-1139}.
\newblock URL \url{https://aclanthology.org/D16-1139}.

\bibitem[Kingma \& Ba(2014)Kingma and Ba]{kingma2014adam}
Kingma, D.~P. and Ba, J.
\newblock Adam: A method for stochastic optimization.
\newblock \emph{arXiv preprint arXiv:1412.6980}, 2014.

\bibitem[Kumar et~al.(2019)Kumar, Fu, Soh, Tucker, and
  Levine]{kumar2019stabilizing}
Kumar, A., Fu, J., Soh, M., Tucker, G., and Levine, S.
\newblock Stabilizing off-policy q-learning via bootstrapping error reduction.
\newblock \emph{Advances in Neural Information Processing Systems}, 32, 2019.

\bibitem[Lee et~al.(2023)Lee, Cho, Kim, Gwak, Kim, Choo, Yun, and
  Yun]{lee2023plastic}
Lee, H., Cho, H., Kim, H., Gwak, D., Kim, J., Choo, J., Yun, S.-Y., and Yun, C.
\newblock Plastic: Improving input and label plasticity for sample efficient
  reinforcement learning.
\newblock In \emph{Thirty-seventh Conference on Neural Information Processing
  Systems}, 2023.

\bibitem[Lee \& Shin(2022)Lee and Shin]{lee2022renyicl}
Lee, K. and Shin, J.
\newblock R{\'e}nyicl: Contrastive representation learning with skew r{\'e}nyi
  divergence.
\newblock \emph{Advances in Neural Information Processing Systems},
  35:\penalty0 6463--6477, 2022.

\bibitem[Lee(2001)]{lee2001effectiveness}
Lee, L.
\newblock On the effectiveness of the skew divergence for statistical language
  analysis.
\newblock In \emph{International Workshop on Artificial Intelligence and
  Statistics}, pp.\  176--183. PMLR, 2001.

\bibitem[Lin et~al.(2020)Lin, Wohlwend, Chen, and
  Lei]{lin-etal-2020-autoregressive}
Lin, A., Wohlwend, J., Chen, H., and Lei, T.
\newblock Autoregressive knowledge distillation through imitation learning.
\newblock In \emph{Proceedings of the 2020 Conference on Empirical Methods in
  Natural Language Processing (EMNLP)}, pp.\  6121--6133, Online, November
  2020. Association for Computational Linguistics.
\newblock \doi{10.18653/v1/2020.emnlp-main.494}.
\newblock URL \url{https://aclanthology.org/2020.emnlp-main.494}.

\bibitem[Lin(2004)]{lin-2004-rouge}
Lin, C.-Y.
\newblock {ROUGE}: A package for automatic evaluation of summaries.
\newblock In \emph{Text Summarization Branches Out}, pp.\  74--81, Barcelona,
  Spain, July 2004. Association for Computational Linguistics.
\newblock URL \url{https://aclanthology.org/W04-1013}.

\bibitem[Liu et~al.(2021)Liu, Pillutla, Welleck, Oh, Choi, and
  Harchaoui]{liu2021divergence}
Liu, L., Pillutla, K., Welleck, S., Oh, S., Choi, Y., and Harchaoui, Z.
\newblock Divergence frontiers for generative models: Sample complexity,
  quantization effects, and frontier integrals.
\newblock \emph{Advances in Neural Information Processing Systems},
  34:\penalty0 12930--12942, 2021.

\bibitem[Loshchilov \& Hutter(2017)Loshchilov and
  Hutter]{loshchilov2017decoupled}
Loshchilov, I. and Hutter, F.
\newblock Decoupled weight decay regularization.
\newblock \emph{arXiv preprint arXiv:1711.05101}, 2017.

\bibitem[Mirzadeh et~al.(2020)Mirzadeh, Farajtabar, Li, Levine, Matsukawa, and
  Ghasemzadeh]{mirzadeh2020improved}
Mirzadeh, S.~I., Farajtabar, M., Li, A., Levine, N., Matsukawa, A., and
  Ghasemzadeh, H.
\newblock Improved knowledge distillation via teacher assistant.
\newblock In \emph{Proceedings of the AAAI conference on artificial
  intelligence}, volume~34, pp.\  5191--5198, 2020.

\bibitem[Mnih et~al.(2015)Mnih, Kavukcuoglu, Silver, Rusu, Veness, Bellemare,
  Graves, Riedmiller, Fidjeland, Ostrovski, et~al.]{mnih2015human}
Mnih, V., Kavukcuoglu, K., Silver, D., Rusu, A.~A., Veness, J., Bellemare,
  M.~G., Graves, A., Riedmiller, M., Fidjeland, A.~K., Ostrovski, G., et~al.
\newblock Human-level control through deep reinforcement learning.
\newblock \emph{nature}, 518\penalty0 (7540):\penalty0 529--533, 2015.

\bibitem[Narayan et~al.(2018)Narayan, Cohen, and Lapata]{narayan2018don}
Narayan, S., Cohen, S.~B., and Lapata, M.
\newblock Don’t give me the details, just the summary! topic-aware
  convolutional neural networks for extreme summarization.
\newblock In \emph{Proceedings of the 2018 Conference on Empirical Methods in
  Natural Language Processing}, pp.\  1797--1807, 2018.

\bibitem[OpenAI(2023)]{openai2023gpt4}
OpenAI.
\newblock Gpt-4 technical report, 2023.

\bibitem[Ouyang et~al.(2022)Ouyang, Wu, Jiang, Almeida, Wainwright, Mishkin,
  Zhang, Agarwal, Slama, Ray, et~al.]{ouyang2022training}
Ouyang, L., Wu, J., Jiang, X., Almeida, D., Wainwright, C., Mishkin, P., Zhang,
  C., Agarwal, S., Slama, K., Ray, A., et~al.
\newblock Training language models to follow instructions with human feedback.
\newblock \emph{Advances in Neural Information Processing Systems},
  35:\penalty0 27730--27744, 2022.

\bibitem[Papineni et~al.(2002)Papineni, Roukos, Ward, and
  Zhu]{papineni-etal-2002-bleu}
Papineni, K., Roukos, S., Ward, T., and Zhu, W.-J.
\newblock {B}leu: a method for automatic evaluation of machine translation.
\newblock In Isabelle, P., Charniak, E., and Lin, D. (eds.), \emph{Proceedings
  of the 40th Annual Meeting of the Association for Computational Linguistics},
  pp.\  311--318, Philadelphia, Pennsylvania, USA, July 2002. Association for
  Computational Linguistics.
\newblock \doi{10.3115/1073083.1073135}.
\newblock URL \url{https://aclanthology.org/P02-1040}.

\bibitem[Pedro(2000)]{pedro2000unified}
Pedro, D.
\newblock A unified bias-variance decomposition and its applications.
\newblock In \emph{17th International Conference on Machine Learning}, pp.\
  231--238, 2000.

\bibitem[Peng et~al.(2023)Peng, Li, He, Galley, and Gao]{peng2023instruction}
Peng, B., Li, C., He, P., Galley, M., and Gao, J.
\newblock Instruction tuning with gpt-4, 2023.

\bibitem[Radford et~al.(2018)Radford, Narasimhan, Salimans, Sutskever,
  et~al.]{radford2018improving}
Radford, A., Narasimhan, K., Salimans, T., Sutskever, I., et~al.
\newblock Improving language understanding by generative pre-training.
\newblock \emph{OpenAI blog}, 2018.

\bibitem[Radford et~al.(2019)Radford, Wu, Child, Luan, Amodei, Sutskever,
  et~al.]{radford2019language}
Radford, A., Wu, J., Child, R., Luan, D., Amodei, D., Sutskever, I., et~al.
\newblock Language models are unsupervised multitask learners.
\newblock \emph{OpenAI blog}, 1\penalty0 (8):\penalty0 9, 2019.

\bibitem[Raffel et~al.(2020)Raffel, Shazeer, Roberts, Lee, Narang, Matena,
  Zhou, Li, and Liu]{raffel2020exploring}
Raffel, C., Shazeer, N., Roberts, A., Lee, K., Narang, S., Matena, M., Zhou,
  Y., Li, W., and Liu, P.~J.
\newblock Exploring the limits of transfer learning with a unified text-to-text
  transformer.
\newblock \emph{The Journal of Machine Learning Research}, 21\penalty0
  (1):\penalty0 5485--5551, 2020.

\bibitem[Rubenstein et~al.(2019)Rubenstein, Bousquet, Djolonga, Riquelme, and
  Tolstikhin]{rubenstein2019practical}
Rubenstein, P., Bousquet, O., Djolonga, J., Riquelme, C., and Tolstikhin, I.~O.
\newblock Practical and consistent estimation of f-divergences.
\newblock \emph{Advances in Neural Information Processing Systems}, 32, 2019.

\bibitem[Sanh et~al.(2019)Sanh, Debut, Chaumond, and
  Wolf]{Sanh2019DistilBERTAD}
Sanh, V., Debut, L., Chaumond, J., and Wolf, T.
\newblock Distilbert, a distilled version of bert: smaller, faster, cheaper and
  lighter.
\newblock \emph{ArXiv}, abs/1910.01108, 2019.
\newblock URL \url{https://api.semanticscholar.org/CorpusID:203626972}.

\bibitem[See et~al.(2017)See, Liu, and Manning]{see-etal-2017-get}
See, A., Liu, P.~J., and Manning, C.~D.
\newblock Get to the point: Summarization with pointer-generator networks.
\newblock In Barzilay, R. and Kan, M.-Y. (eds.), \emph{Proceedings of the 55th
  Annual Meeting of the Association for Computational Linguistics (Volume 1:
  Long Papers)}, pp.\  1073--1083, Vancouver, Canada, July 2017. Association
  for Computational Linguistics.
\newblock \doi{10.18653/v1/P17-1099}.
\newblock URL \url{https://aclanthology.org/P17-1099}.

\bibitem[Sun et~al.(2019)Sun, Cheng, Gan, and Liu]{sun-etal-2019-patient}
Sun, S., Cheng, Y., Gan, Z., and Liu, J.
\newblock Patient knowledge distillation for {BERT} model compression.
\newblock In Inui, K., Jiang, J., Ng, V., and Wan, X. (eds.), \emph{Proceedings
  of the 2019 Conference on Empirical Methods in Natural Language Processing
  and the 9th International Joint Conference on Natural Language Processing
  (EMNLP-IJCNLP)}, pp.\  4323--4332, Hong Kong, China, November 2019.
  Association for Computational Linguistics.
\newblock \doi{10.18653/v1/D19-1441}.
\newblock URL \url{https://aclanthology.org/D19-1441}.

\bibitem[Taori et~al.(2023)Taori, Gulrajani, Zhang, Dubois, Li, Guestrin,
  Liang, and Hashimoto]{alpaca}
Taori, R., Gulrajani, I., Zhang, T., Dubois, Y., Li, X., Guestrin, C., Liang,
  P., and Hashimoto, T.~B.
\newblock Stanford alpaca: An instruction-following llama model.
\newblock \url{https://github.com/tatsu-lab/stanford_alpaca}, 2023.

\bibitem[Touvron et~al.(2023)Touvron, Martin, Stone, Albert, Almahairi, Babaei,
  Bashlykov, Batra, Bhargava, Bhosale, et~al.]{touvron2023llama2}
Touvron, H., Martin, L., Stone, K., Albert, P., Almahairi, A., Babaei, Y.,
  Bashlykov, N., Batra, S., Bhargava, P., Bhosale, S., et~al.
\newblock Llama 2: Open foundation and fine-tuned chat models.
\newblock \emph{arXiv preprint arXiv:2307.09288}, 2023.

\bibitem[Wang et~al.(2023{\natexlab{a}})Wang, Wang, Li, Gao, Yin, and
  Ren]{wang2023scott}
Wang, P., Wang, Z., Li, Z., Gao, Y., Yin, B., and Ren, X.
\newblock Scott: Self-consistent chain-of-thought distillation.
\newblock \emph{arXiv preprint arXiv:2305.01879}, 2023{\natexlab{a}}.

\bibitem[Wang et~al.(2022)Wang, Mishra, Alipoormolabashi, Kordi, Mirzaei, Naik,
  Ashok, Dhanasekaran, Arunkumar, Stap, Pathak, Karamanolakis, Lai, Purohit,
  Mondal, Anderson, Kuznia, Doshi, Pal, Patel, Moradshahi, Parmar, Purohit,
  Varshney, Kaza, Verma, Puri, Karia, Doshi, Sampat, Mishra, Reddy~A, Patro,
  Dixit, and Shen]{wang-etal-2022-super}
Wang, Y., Mishra, S., Alipoormolabashi, P., Kordi, Y., Mirzaei, A., Naik, A.,
  Ashok, A., Dhanasekaran, A.~S., Arunkumar, A., Stap, D., Pathak, E.,
  Karamanolakis, G., Lai, H., Purohit, I., Mondal, I., Anderson, J., Kuznia,
  K., Doshi, K., Pal, K.~K., Patel, M., Moradshahi, M., Parmar, M., Purohit,
  M., Varshney, N., Kaza, P.~R., Verma, P., Puri, R.~S., Karia, R., Doshi, S.,
  Sampat, S.~K., Mishra, S., Reddy~A, S., Patro, S., Dixit, T., and Shen, X.
\newblock Super-{N}atural{I}nstructions: Generalization via declarative
  instructions on 1600+ {NLP} tasks.
\newblock In \emph{Proceedings of the 2022 Conference on Empirical Methods in
  Natural Language Processing}, pp.\  5085--5109, Abu Dhabi, United Arab
  Emirates, December 2022. Association for Computational Linguistics.
\newblock \doi{10.18653/v1/2022.emnlp-main.340}.
\newblock URL \url{https://aclanthology.org/2022.emnlp-main.340}.

\bibitem[Wang et~al.(2023{\natexlab{b}})Wang, Kordi, Mishra, Liu, Smith,
  Khashabi, and Hajishirzi]{wang-etal-2023-self-instruct}
Wang, Y., Kordi, Y., Mishra, S., Liu, A., Smith, N.~A., Khashabi, D., and
  Hajishirzi, H.
\newblock Self-instruct: Aligning language models with self-generated
  instructions.
\newblock In \emph{Proceedings of the 61st Annual Meeting of the Association
  for Computational Linguistics (Volume 1: Long Papers)}, pp.\  13484--13508,
  Toronto, Canada, July 2023{\natexlab{b}}. Association for Computational
  Linguistics.
\newblock \doi{10.18653/v1/2023.acl-long.754}.
\newblock URL \url{https://aclanthology.org/2023.acl-long.754}.

\bibitem[Wen et~al.(2023)Wen, Li, Du, and Mou]{wen-etal-2023-f}
Wen, Y., Li, Z., Du, W., and Mou, L.
\newblock f-divergence minimization for sequence-level knowledge distillation.
\newblock In \emph{Proceedings of the 61st Annual Meeting of the Association
  for Computational Linguistics (Volume 1: Long Papers)}, pp.\  10817--10834,
  Toronto, Canada, July 2023. Association for Computational Linguistics.
\newblock \doi{10.18653/v1/2023.acl-long.605}.
\newblock URL \url{https://aclanthology.org/2023.acl-long.605}.

\bibitem[Xue et~al.(2021)Xue, Constant, Roberts, Kale, Al-Rfou, Siddhant,
  Barua, and Raffel]{xue-etal-2021-mt5}
Xue, L., Constant, N., Roberts, A., Kale, M., Al-Rfou, R., Siddhant, A., Barua,
  A., and Raffel, C.
\newblock m{T}5: A massively multilingual pre-trained text-to-text transformer.
\newblock In Toutanova, K., Rumshisky, A., Zettlemoyer, L., Hakkani-Tur, D.,
  Beltagy, I., Bethard, S., Cotterell, R., Chakraborty, T., and Zhou, Y.
  (eds.), \emph{Proceedings of the 2021 Conference of the North American
  Chapter of the Association for Computational Linguistics: Human Language
  Technologies}, pp.\  483--498, Online, June 2021. Association for
  Computational Linguistics.
\newblock \doi{10.18653/v1/2021.naacl-main.41}.
\newblock URL \url{https://aclanthology.org/2021.naacl-main.41}.

\bibitem[Zhang et~al.(2022)Zhang, Roller, Goyal, Artetxe, Chen, Chen, Dewan,
  Diab, Li, Lin, et~al.]{zhang2022opt}
Zhang, S., Roller, S., Goyal, N., Artetxe, M., Chen, M., Chen, S., Dewan, C.,
  Diab, M., Li, X., Lin, X.~V., et~al.
\newblock Opt: Open pre-trained transformer language models.
\newblock \emph{arXiv preprint arXiv:2205.01068}, 2022.

\bibitem[Zheng et~al.(2023)Zheng, Chiang, Sheng, Zhuang, Wu, Zhuang, Lin, Li,
  Li, Xing, Zhang, Gonzalez, and Stoica]{zheng2023judging}
Zheng, L., Chiang, W.-L., Sheng, Y., Zhuang, S., Wu, Z., Zhuang, Y., Lin, Z.,
  Li, Z., Li, D., Xing, E., Zhang, H., Gonzalez, J.~E., and Stoica, I.
\newblock Judging {LLM}-as-a-judge with {MT}-bench and chatbot arena.
\newblock In \emph{Thirty-seventh Conference on Neural Information Processing
  Systems Datasets and Benchmarks Track}, 2023.
\newblock URL \url{https://openreview.net/forum?id=uccHPGDlao}.

\end{thebibliography}
\bibliographystyle{icml2024}

\clearpage
{
    \newpage
   \twocolumn[
    \centering
    \Large
    \vspace{1.0em}
    \textbf{\mytitle} \\
    \vspace{0.5em}Supplementary Material \\
    \vspace{1.0em}
    ]\
}
    
\appendix
\vspace{-25pt}
\section{Limitation}
While \alg shows effectiveness in terms of computational efficiency of training and students' performances compared to recent baselines~\cite{gu2023knowledge, agarwal2023gkd}, we acknowledge a couple of limitations:
\begin{itemize}[leftmargin=*, itemsep=0pt]
\vspace{-5pt}
    \item \textbf{KLD-based Objective}: Our focus is mainly on (R)KLD because of its tractability that allows decomposing sequence-level distillation into token-wise distillation (as shown in Eq.~(\ref{eq:approx_kld})-(\ref{eq:token})). However, (R)KLD also entails limitations previously identified in \citet{ji2023tailoring, ren2024emo}: mode averaging (or mode collapse) and train-inference mismatch. Although DistiLLM, especially S(R)KL, effectively mitigated these issues, additionally combining with objective functions based on TVD~\cite{ji2023tailoring} or EMD~\cite{ren2024emo} could further enhance performance. We expect that by linearly interpolating between SKL and EMD (or TVD), we can leverage the strengths of both: the rapid convergence of SKL during the early stages of training and the superior performance of EMD or TVD towards the late training compared to KLD.
    \item \textbf{Supervised Fine-tuning Approach}: \alg is is primarily designed for a supervised fine-tuning setup, which has recently shown effectiveness~\cite{chen2024selfplay}, Meanwhile, many contemporary chat LLMs utilize preference optimization~\cite{ouyang2022training, rafailov2024direct}; thus, an extension of \alg to accommodate human preference optimization setups might be considered as future work.
    \item \textbf{Same Tokenizer between Teacher \& Student}: Our method focuses on the scenario where the teacher and student models share the same tokenizer, which is also a common setup for white-box KD methods. However, with recent techniques~\cite{wan2024knowledge, boizard2024towards} that facilitate the transfer of knowledge between models with different tokenizers, we can designate this for future work to be explored in conjunction with these new methods.
\end{itemize}

\section{Further Discussion on Skew KLD}
Here, we provide the derivation of our theoretical results described in Sec.~\ref{sec:skew_kl} and further (empirical) discussion on our proposed SKL.

\subsection{Details for Gradient Analysis}\label{app:gradient}
\paragraph{Derivation of Sec.~\ref{sec:skew_kl}.} We derive the sample-wise gradient of KLD, SKL, RKLD, and SRKL to support our argument that simple skew operation on KLD improves the stability of the optimization.
\begin{itemize}
    \item \textit{Firstly}, we compute the gradient of KLD for a single sample $(\mathbf{x}, \mathbf{y})$:
    \begin{align*}
        \nabla_{\theta} D_{\text{KL}} (p, q_{\theta}) &= -\nabla_{\theta} p(\mathbf{y}|\mathbf{x}) \log q_{\theta} (\mathbf{y}|\mathbf{x}) \\
        &= -\frac{p(\mathbf{y}|\mathbf{x})}{q_{\theta}(\mathbf{y}|\mathbf{x})} \nabla_{\theta} q_{\theta} (\mathbf{y}|\mathbf{x}).
    \end{align*}
    \item \textit{Secondly}, we compute the gradient of SKL for $(\mathbf{x}, \mathbf{y})$:
    \begin{align*}
        \nabla_{\theta} &D_{\text{SKL}}^{(\alpha)} (p, q_{\theta}) \\
        &= - \nabla_{\theta} p(\mathbf{y}|\mathbf{x}) \log \left( \alpha p(\mathbf{y}|\mathbf{x}) + (1-\alpha) q_{\theta} (\mathbf{y}|\mathbf{x}) \right) \\
        &= - \nabla_{\theta} p(\mathbf{y}|\mathbf{x}) \log \tilde{q}_{\theta} (\mathbf{y}|\mathbf{x}) \\
        &= - \frac{p(\mathbf{y}|\mathbf{x})}{\tilde{q}_{\theta} (\mathbf{y}|\mathbf{x})} \nabla_{\theta} \tilde{q}_{\theta} (\mathbf{y}|\mathbf{x}) \\
        &= - \frac{p(\mathbf{y}|\mathbf{x})}{\tilde{q}_{\theta} (\mathbf{y}|\mathbf{x})} \cdot (1-\alpha) \cdot \nabla_{\theta} q_{\theta} (\mathbf{y}|\mathbf{x}).
    \end{align*}
    \item \textit{Thirdly}, we compute the gradient of RKLD for $(\mathbf{x}, \mathbf{y})$:
    \begin{align*}
        \nabla_{\theta} &D_{\text{KL}} (q_{\theta}, p) \\
        &= \log p(\mathbf{y}|\mathbf{x}) \nabla_{\theta} q_{\theta} (\mathbf{y}|\mathbf{x}) - \nabla_{\theta}\!\left( q_{\theta} (\mathbf{y}|\mathbf{x}) \log q_{\theta} (\mathbf{y}|\mathbf{x}) \right) \\
        &= \nabla_{\theta} q_{\theta} (\mathbf{y}|\mathbf{x}) \cdot (\log {p(\mathbf{y}|\mathbf{x})}{q(\mathbf{y}|\mathbf{x})} - 1).
    \end{align*}
    \item \textit{Lastly}, with a definition of $\tilde{p}(\mathbf{y}|\mathbf{x}) = (1-\alpha) p(\mathbf{y}|\mathbf{x}) + \alpha q_{\theta}(\mathbf{y}|\mathbf{x})$, we compute the gradient of SRKL for $(\mathbf{x}, \mathbf{y})$:
    \begin{align*}
        &\nabla_{\theta} D_{\text{SKL}}^{(\alpha)} (q_{\theta}, p) \\
        &= \nabla_{\theta} (q_{\theta} (\mathbf{y}|\mathbf{x}) \log \tilde{p} (\mathbf{y}|\mathbf{x})) - \nabla_{\theta} (q_{\theta} (\mathbf{y}|\mathbf{x}) \log q_{\theta} (\mathbf{y}|\mathbf{x})) \\
        &= \log \tilde{p}(\mathbf{y}|\mathbf{x}) \nabla_{\theta} q_{\theta} (\mathbf{y}|\mathbf{x})  + \frac{q_{\theta} (\mathbf{y}|\mathbf{x})}{\tilde{p}(\mathbf{y}|\mathbf{x})} \nabla_{\theta} \tilde{p}(\mathbf{y}|\mathbf{x}) \\
        &- \nabla_{\theta} q_{\theta} (\mathbf{y}|\mathbf{x}) \log q_{\theta} (\mathbf{y}|\mathbf{x}) - \nabla_{\theta} q_{\theta} (\mathbf{y}|\mathbf{x})\\
        &= -\left( \log \frac{q_{\theta}(\mathbf{y}|\mathbf{x})}{\tilde{p}(\mathbf{y}|\mathbf{x})} + 1 - \alpha \frac{q_{\theta}(\mathbf{y}|\mathbf{x})}{\tilde{p}(\mathbf{y}|\mathbf{x})} \right) \nabla_{\theta} q_{\theta} (\mathbf{y}|\mathbf{x}),
    \end{align*}
\end{itemize}

As we described in Section~\ref{sec:skew_kl}, we can prevent the undesired gradient norm explosion due to smoothed distributions $\tilde{q}_{\theta}(\mathbf{y}|\mathbf{x})$ and $\tilde{p}(\mathbf{y}|\mathbf{x})$ for SKL and SRKL, respectively.

\begin{figure*}[t]
    \centering
    \includegraphics[width=1.0\linewidth]{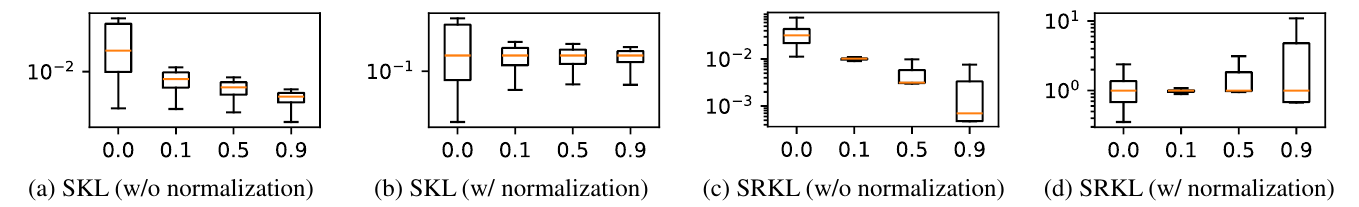}
    \vspace{-17.5pt}
    \caption{Gradient coefficient distribution for SKL and SRKL across different skew values, $\alpha$. Skewing KLD and RKLD effectively smooth the gradient norm, as seen in (a) and (c). For coefficients normalized by their median value, SKL shows a similar distribution when $\alpha>0$ while SRKL exhibits explosion, as depicted in (b) and (d).}
    \label{fig:norm_grad}
\end{figure*}
\paragraph{Further empirical discussion.} Recent optimizers, such as Adam~\cite{kingma2014adam} and AdamW~\cite{loshchilov2017decoupled}, adaptively adjust learning rates based on the gradient scale. 
Hence, it is also important to observe the distribution of the normalized gradient divided by their scale.
In Fig.~\ref{fig:norm_grad}, we additionally plot the normalized version of the gradient coefficient, obtained by dividing the corresponding median values of the coefficient. 
SKL shows consistently reduced variance as $\alpha$ increases not only for original values but also normalized values.
However, SRKL shows the smallest variance at $\alpha$ of 0.1 for both original and normalized values, but when $\alpha$ is larger than 0.1, the variance becomes larger as $\alpha$ grows.
These results are also related to those in Fig.~\ref{fig:fig_ablation_alpha}, where SKL shows higher robustness than SRKL across different $\alpha$ values.
Still, our mathematical analysis is valid given that the gradient scale becomes smaller as $\alpha$ increases.

\subsection{Proof of Theorem 1}\label{app:theory}
\textbf{Theorem 1 (Restated).} \textit{Assume $p^{1}$, $p^{2}$ be two probability distributions such that $\mathbb{V}_{p^{2}}[dp^{1}/dp^{2}] < \infty$ and $p^{1}$ is absolutely continuous with respect to $p^{2}$~(i.e., the Radon-Nikodym derivative $dp^{1}/dp^{2}$ exists). Then, for $\forall \alpha < 1/8$, the variance of $\alpha$-SKL estimator satisfies
\begin{align*}
    \mathbb{E}[|D_{\text{SKL}}^{(\alpha)}&(p^{1}_n, p^{2}_n) - D_{\text{SKL}}^{(\alpha)}(p^{1}, p^{2})|^{2}] \\
        &\leq \frac{c_{1}(\alpha)}{n^{2}} + \frac{c_{2} \log^{2}(\alpha n)}{n} + \frac{c_3 \log^{2}(c_4 n)}{\alpha^{2}n},
\end{align*}
for $c_{1}(\alpha) = \min\{\frac{1}{\alpha^2}, \frac{\chi^{2}(p^{1}, p^{2})^2}{(1-\alpha)^2}\}$ and constants $c_{2}, c_{3}, c_{4} > 0$ that are independent of $n$, $\alpha$, and $D(p^{1}, p^{2})$, where $\chi^{2}(p^{1}, p^{2}) \coloneqq \mathbb{E}_{p^{2}} [(dp^{1}/dp^{2})^{2}]$.}
Here, we denote $p^{1} \ll p^{2}$ as $p^{1}$ is absolutely continuous with respect to $p^{2}$.

In this section, we state and define the regularity assumptions to derive the asymptotic upper bounds for the variance of $\alpha$-skew KL divergence. Formally, the $f$-divergence of two distributions is defined as 
\begin{equation*}
    D_{f} (p^{1}, p^{2}) = \mathbb{E}_{\mathbf{y} \sim p^{1}} \left[ f\left(\frac{p^{1}(\mathbf{y}|\mathbf{x})}{p^{2}(\mathbf{y}|\mathbf{x})}\right) \right] \coloneqq \mathbb{E}_{p^{1}} \left[ f\left(\frac{dp^{1}}{dp^{2}}\right) \right],
\end{equation*}
where $dp^{1}$ and $dp^{2}$ are the probability densities of probability $p^{1}$ and $p^{2}$.  


\noindent
The KL divergence is a $f$-divergence generated by $f(t) = t \log t - t + 1$, and the $\alpha$-skew KL divergence is a $f$-divergence generated by
\begin{equation*}
    f^{(\alpha)}(t) = t \log \left( \frac{t}{\alpha t + 1 - \alpha} \right) - (1-\alpha)(t-1).
\end{equation*}

By following \citet{liu2021divergence} and \citet{lee2022renyicl}, we state the following regularity assumptions on the functions $f$ and $f^*$.

\begin{assumption}\label{thm:assum1}
    The generator $f$ is twice continuously differentiable with $f'(1)=0$. Moreover
    \begin{enumerate}[label={(\bfseries A\arabic*):}] 
        \item We have $C_{0} \coloneqq f(0) < \infty$ and $C_{0}^{*} \coloneqq f^{*}(0) < \infty$.
        \item There exist constants $C_1, C_1^* < \infty$ such that for any $t \in (0,1)$, we have $|f'(t)| \leq C_{1} \max\{1, \log(1/t) \}$ and $|(f^{*})'(t)| \leq C_{1}^{*} \max \{ 1, \log (1/t)\}$.
        \item There exist constants $C_{2}, C_{2}^{*} < \infty$ such that for every $t \in (0, \infty)$, we have $\frac{t}{2} f''(t) \leq C_2$ and $\frac{t}{2} (f^{*})''(t) \leq C_2^*$.
    \end{enumerate}
\end{assumption}

\noindent
One can observe that both KLD and RKLD do not satisfy Assumption~\ref{thm:assum1} because KLD is unbounded. On the other hand, the $\alpha$-skew KL divergence satisfies Assumption~\ref{thm:assum1} from the following proposition.

\begin{lemma}[\citealt{liu2021divergence}]\label{thm:lemma1}
    The $\alpha$-skew KL divergence generated by $f^{(\alpha)}$ satisfies Assumption~\ref{thm:assum1} with
    \begin{align*}
        C_0 = 1 - \alpha, \;\; C_0^* = \log \frac{1}{\alpha} - 1 + \alpha, \;\; C_1 = 1, \\
        C_1^* = \frac{(1-\alpha)^{2}}{\alpha}, \;\; C_2 = \frac{1}{2}, \;\; C_2^* = \frac{1-\alpha}{8\alpha}.
    \end{align*}
\end{lemma}

\noindent
For general $f$-divergences which satisfy Assumption~\ref{thm:assum1}, the following concentration bound holds.

\begin{lemma}[\citealt{liu2021divergence}]\label{thm:lemma2}
    Assume $f$ satisfies Assumption~\ref{thm:assum1}, and let $p^{1}$ and $p^{2}$ be two distributions with $p^{1} \ll p^{2}$. Let $p^{1}_{m}$ be $m$ i.i.d samples from $p^{1}$ and $p^{2}_n$ be $n$ i.i.d samples from $p^{2}$. Then the $f$-divergence $D_f$ satisfies following:
    \begin{align*}
        &\mathbb{P}[|D_f(p^{1}_{m}, p^{2}_{n}) - \mathbb{E}[D_f(p^{1}_{m}, p^{2}_{n})]| > \epsilon] \leq \\
        &2 \exp \left(-\frac{\epsilon^{2}}{\frac{2}{m}(C_1 \log m + c_1)^{2} + \frac{2}{n} (C_1^* \log n + c_2)^{2}} \right)
    \end{align*}
    where $c_1 = \max \{ C_0^*, C_2\}$ and $c_2 = \max \{ C_0, C_2^* \}$.
\end{lemma}

Thus, the following lemma derives a concentration bound for the $\alpha$-skew KL divergence by plugging the constants in Lemma~\ref{thm:lemma1} to Lemma~\ref{thm:lemma2}.

\begin{lemma}[\citealt{lee2022renyicl}]\label{thm:lemma3}
    For $\alpha < \frac{1}{8}$, the following holds:
    \begin{align*}
        &\mathbb{P}[|D_{\text{SKL}}^{(\alpha)}(p^{1}_{m}, p^{2}_{n}) - \mathbb{E}[D_{\text{SKL}}^{(\alpha)}(p^{1}_{m}, p^{2}_{n})]| > \epsilon] \\
        &\leq 2 \exp \left(-\frac{\epsilon^{2}}{\frac{2}{m} \log^{2}(\alpha m) + \frac{2}{\alpha^{2} n} \log^{2} (e^{1/8}n)} \right)
    \end{align*}
\end{lemma}
\begin{proof}
    Note that $C_{0}^{*} = \log (1/\alpha) - 1 + \alpha \geq C_2 = 1/2$, and $C_0 = 1 - \alpha \leq C_2^* = \frac{1-\alpha}{8\alpha}$ for $\alpha < \frac{1}{8}$. Then the concentration bound follows from Lemma~\ref{thm:lemma3}.
\end{proof}

Lastly, we present the following upper bound on the bias of the empirical estimator of KLD.

\begin{lemma}[\citealt{rubenstein2019practical}]\label{thm:lemma4}
    Suppose $p^{1} \ll p^{2}$, and $\mathbb{V}_{p^{1}}[dp^{1}/dp^{2}] < \infty$. Then we have
    \begin{equation*}
        |\mathbb{E}[D_{\text{KL}}(p^{1}_{m}, p^{2}_{n})] - D_{KL} (p^{1}, p^{2}) | \leq \frac{\chi^{2}(p^{1}, p^{2})}{\min \{ n, m \}}.
    \end{equation*}
\end{lemma}

\begin{lemma}[\citealt{lee2022renyicl}]\label{thm:lemma5} %
    For $\alpha \in (0, 1)$, the following holds:
    \begin{equation*}
        |\mathbb{E}[D_{\text{SKL}}^{(\alpha)}(p^{1}_{m}, p^{2}_{n})] - D_{\text{SKL}}^{(\alpha)} (p^{1}, p^{2}) | \leq \frac{c(\alpha)}{\min \{ n, m \}},
    \end{equation*}
    where $c(\alpha) \coloneqq \min \left\{ \frac{1}{\alpha}, \frac{\chi^{2}(p^{1}_{m}, p^{2}_{n})}{1-\alpha} \right\}$.
\end{lemma}

\begin{proof}
    From Lemma~\ref{thm:lemma4}, we have 
    \begin{equation*}
        |\mathbb{E}[D_{\text{SKL}}^{(\alpha)}(p^{1}_{m}, p^{2}_{n})] - D_{\text{SKL}}^{(\alpha)} (p^{1}, p^{2}) | \leq \frac{\chi^{2}(p^{1}, \alpha p^1 + (1-\alpha) p^2}{\min \{ n, m \}},
    \end{equation*}
    where $\chi^2 (p^1, \alpha p^1 + (1-\alpha) p^2) = \int \frac{d^2 p^1}{\alpha d p^1 + (1-\alpha) dp^2} \leq \int \frac{1}{\alpha} dp^1 = \frac{1}{\alpha}$, or $\int \frac{d^2 p^1}{\alpha d p^1 + (1-\alpha) dp^2} \leq \frac{1}{1-\alpha} \int \frac{d^2 p^1}{dp^2} = \frac{\chi^2 (p^1, p^2)}{1-\alpha}$. Therefore, we have
    \begin{equation*}
        |\mathbb{E}[D_{\text{SKL}}^{(\alpha)}(p^{1}_{m}, p^{2}_{n})] - D_{\text{SKL}}^{(\alpha)} (p^{1}, p^{2}) | \leq \frac{c(\alpha)}{\min \{n, m\}}
    \end{equation*}
    for $c(\alpha) \coloneqq \min \left\{ \frac{1}{\alpha}, \frac{\chi^2 (p^1, p^2)}{1-\alpha} \right\}$
\end{proof}

Now we present the proof of Theorem~\ref{method:thm} in the main paper.

\begin{proof}

    Define
    \begin{align*}
        &B_1 \coloneqq D_{\text{SKL}}^{(\alpha)}(p^{1}_{n}, p^{2}_n) - \mathbb{E}[D_{\text{SKL}}^{(\alpha)}(p^{1}_{n}, p^{2}_n)] \\
        &B_2 \coloneqq \mathbb{E}[D_{\text{SKL}}^{(\alpha)}(p^{1}_{n} \| p^{2}_{n})] - D_{\text{SKL}}^{(\alpha)} (p^{1}, p^{2}).
    \end{align*}

    Then, by using bias-variance decomposition~\cite{pedro2000unified}, we have 
    \begin{align*}
        \mathbb{E}_{p^{1}, p^{2}} [ | D_{\text{SKL}}^{(\alpha)}(p^{1}_n, p^{2}_n) - D_{\text{SKL}}^{(\alpha)} (p^{1}, p^{2}) |^{2} ] \\
        = \underbrace{\mathbb{E}_{p^{1}, p^{2}} [ | B_{1} |^{2} ]}_{\text{Variance}} + \underbrace{\mathbb{E}_{p^{1}, p^{2}} [ | B_{2} |^{2} ]}_{\text{Bias}^{2}}.
    \end{align*}
    
    Since the following holds for any random variable $X$,
    \begin{align*}
        \mathbb{V}(X) &= \mathbf{E}[(X - \mathbb{E}X)^{2}] \\
        &= \int_{0}^{\infty} \mathbb{P} [| X - \mathbb{E}X|^{2} > t] dt \\
        &= \int_{0}^{\infty} \mathbb{P} [|X - \mathbb{E}X| > \sqrt{t}] dt,
    \end{align*}

    by Lemma~\ref{thm:lemma2}, we have a variance for estimator $D_{\text{SKL}}^{(\alpha)}(p^{1}_n, p^{2}_n)$ as follows:
    \begin{equation}\label{eq:var}
    \begin{split}
        &\mathbb{V}_{p^{1}, p^{2}}[D_{\text{SKL}}^{(\alpha)}(p^{1}_n, p^{2}_n)] \leq \\
        &\int_{0}^{\infty} 2 \exp \left(-\frac{t}{\frac{2}{n} \left( \log^{2}(\alpha n) + \frac{1}{\alpha^{2}} \log^{2} (e^{1/8}n) \right)}\right) dt.
    \end{split}
    \end{equation}

    As we can directly compute the bias term through Lemma~\ref{thm:lemma5}, we have
    \begin{align}
        \mathbb{E}_{p^{1}, p^{2}} [ | D_{\text{SKL}}^{(\alpha)}(p^{1}_n, p^{2}_n)] - D_{\text{SKL}}^{(\alpha)} (p^{1}, p^{2}) |^{2} ] \notag \\
        \leq \underbrace{\frac{c_{1}(\alpha)}{n^{2}}}_{\text{from Lemma~\ref{thm:lemma5}}} + \underbrace{\frac{c_{2} \log^{2}(\alpha n)}{n} + \frac{c_{3} \log^{2}(c_{4}n)}{\alpha^{2}n}}_{\text{from Eq.~(\ref{eq:var})}},
    \end{align}
    where $c_{1}(\alpha) = \min\left\{\frac{1}{\alpha^{2}}, \frac{\chi^{2}(p^{1}, p^{2})^{2}}{(1-\alpha)^{2}}\right\}$ and constants $c_{2}, c_{3}, c_{4} > 0$ that are independent of $n$, $\alpha$, and $D_{\text{KL}}(p^{1}, p^{2})$, where $\chi^{2}(p^{1}, p^{2}) := \mathbb{E}_{p^{2}} \left[ (dp^{1}/dp^{2})^{2} \right]$.
\end{proof}


\section{Details of \alg Algorithm}\label{app:algorithm}
We describe \alg in detail, especially for the adaptive off-policy approach, which could not be fully explained in Sec.~\ref{sec:adaptive} due to lack of margin. 

Instead of using $\phi$ defined in Sec.~\ref{sec:adaptive}, we further define the probability of using SGOs determined at each validation iteration $\tilde{t}$, denoted as $\phi_{\tilde{t}}$. We adjust the probability by using the following rule with validation loss $\mathcal{L}_{\tilde{t}}$ and $\mathcal{L}_{\tilde{t}-1}$ for iteration $\tilde{t}$ and $\tilde{t}-1$, respectively:
\begin{equation*}
    \phi_{\tilde{t}} = 
    \begin{cases}
        \phi_{\tilde{t}-1} & \text{if } \mathcal{L}_{\tilde{t}} \leq \mathcal{L}_{\tilde{t}-1} + \varepsilon  \\
        \min (\phi_{\tilde{t}-1} + 1 / N_{\phi}, 1.0)& \text{otherwise}
    \end{cases},
\end{equation*}
where $\varepsilon$ is the loss tolerance, a hyperparameter introduced to mitigate unexpected fluctuations in loss values and enhance the stability of the process. We set this value as 0.1 for all experiments. $N_{\phi}$ denotes the total number of stages for adjusting the probability, allowing the probability to adopt any value in the set 
$\left\{\frac{i}{N_{\phi}} \,\big| \,\, i = 0, 1, \ldots, N_{\phi} \right\}$.
In our experiments, we set an initial validation loss $\mathcal{L}_{0}$ as validation loss of initialized student models, $\phi_{0}$ as 0.0, and $N_{\phi}$ as 10. If a newly computed validation loss exceeds the preceding one, we increment $\phi_{\tilde{t}}$ by one stage. 

\begin{figure}[t]
    \centering
    \includegraphics[width=1.0\linewidth]{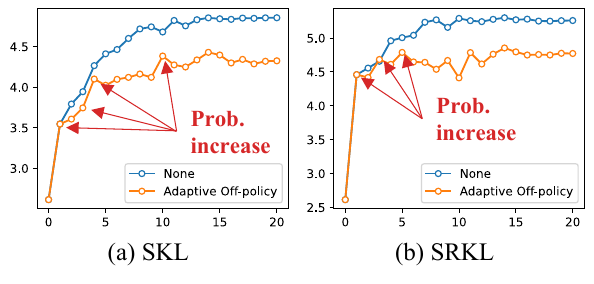}
    \vspace{-27pt}
    \caption{Plot of validation loss values (y-axis) across each validation iteration (x-axis). Although validation loss tends to increase as training progresses, employing SGO effectively prevents this increase. This is the core philosophy of our adaptive SGO scheduler~(\textcolor{orange}{orange line}).}
    \label{fig:fig_val_loss}
\end{figure}
Here, we present empirical evidence to justify using validation loss as the primary metric. Fig.\,\ref{fig:fig_val_loss} shows that the validation loss for both SKL and SRKL consistently increases over iterations, as depicted by the blue lines. This trend indicates that the student models may overly fit to the training dataset and overlook other general sequences, leading to a training-inference mismatch. To create a more versatile student model that performs well on both training and validation datasets, we introduce additional SGOs whenever an increase in validation loss is observed. This approach effectively mitigates the mismatch between training and inference, as well as reduces the risk of noisy feedback.

\section{Experimental Setup}\label{app:setup}
We elaborate the detailed experimental setup regarding the datasets used~(Sec.~\ref{app:dataset}), training details~(Sec.~\ref{app:training_details}), and evaluation details~(Sec.~\ref{app:evaluation_details}).

\subsection{Dataset Description}\label{app:dataset}
We apply \alg on various generation tasks including task-agnostic instruction-following, text summarization, and machine translation. We provide detailed descriptions of the datasets used.
\begin{itemize}[leftmargin=*, itemsep=0pt]
\vspace{-10pt}
    \item \textbf{\texttt{databricks-dolly-15k}}\,{\small(instruction-following, \citealt{DatabricksBlog2023DollyV2})}: {\small\texttt{databricks-dolly-15k}} is an open-source dataset of instruction-following records generated by thousands of Databricks employees in several behavioral categories that are outlined in \citet{ouyang2022training}, including brainstorming, classification, closed QA, generation, information extraction, open QA, and summarization.
    \item \textbf{Self-Instruct}\,{\small(instruction-following, \citealt{wang-etal-2022-super})}: Self-Instruct is a framework designed to enhance the language model's instruction-following capabilities by leveraging the model's own outputs to generate a vast set of instructional data. It includes 52K instructions and 82K inputs and outputs for tuning, along with 252 expert-written tasks aimed at practical applications and additional 50K examples from public datasets for benchmarking.
    \item \textbf{Vicuna}\,{\small(instruction-following, \citealt{vicuna2023})}: We also use 80 challenging questions that were used for evaluating Vicuna, following \citet{gu2023knowledge}.
    \item \textbf{Super-Natural Instruction}\,{\small(instruction-following, \citealt{wang-etal-2022-super})}: Super-Natural Instruction is introduced as a benchmark of 1,616 diverse NLP tasks and their expert-written instructions. The collection covers 76 distinct task types. Its test set consists of 9K samples ranging from 119 tasks.
    \item \textbf{Unnatural Instruction}\,{\small(instruction-following, \citealt{honovich-etal-2023-unnatural})}: Unnatural Instruction uses AI to create 240K instructions with little human help, showing that AI-made data can be as good as human-made data for training language models. The core set of this dataset contains 60K samples.
    \item \textbf{SAMSum}\,(text summarization, \citealt{gliwa2019samsum}): SAMSum consists of 16K messenger-like conversations, annotated with a summary for providing a concise overview of the conversation's content by the third person.
    \item \textbf{XSum}\,{\small(text summarization, \citealt{narayan2018don})}: XSum comprises over 200K news articles, each accompanied by a one-sentence summary designed for the evaluation of abstractive single-document summarization systems, focusing on extreme summarization to capture the essence of articles in a single sentence.
    \item \textbf{CNN/DM}\,{\small(text summarization,\,\citealt{see-etal-2017-get})}:\,CNN/ DM consists of over 300K English news articles that were originally designed for machine-reading and comprehension as well as abstractive question answering, but it now also supports extractive and abstractive summarization.
    \item \textbf{IWSLT 2017}\,{\small(machine translation, \citealt{cettolo-etal-2017-overview})}: IWSLT 2017 addresses text translation, using a single machine translation system for multiple language directions such as English and German. Here, we specifically focus on an English-to-German\,(En-De) translation task.
\end{itemize}

\subsection{Training Details}\label{app:training_details}
For training the teacher and student models, we used four A100 40GB GPUs for the instruction-following task and four RTX 3090 GPUs for the text summarization and machine translation tasks.

\vspace{-7.5pt}
\paragraph{Instruction-following experiments.} Our experimental setup for training LMs on \texttt{databricks-dolly-15k} primarily follows the experimental setup for \citet{gu2023knowledge}. For models within 1B parameters, we search for the learning rates in \{5e-4, 1e-4, 5e-5\}, the batch sizes in \{8, 16, 32\} within the possible maximum batch size for A100 40GB GPUs, and train these models for 20 epochs. For models that have more than 1B parameters, we search for the learning rate in \{5e-5, 1e-5, 5e-6\}, the batch sizes of 8, and train these models for 10 epochs. We fully use the distillation loss for the instruction-following dataset and language modeling loss for OpenWebText~\cite{Gokaslan2019OpenWeb} corpus. The checkpoints of each student are selected by the ROUGE-L scores on the validation set. 

For MiniLLM\,\cite{gu2023knowledge}, we follow the original setup except for the number of GPUs. For ImitKD\,\cite{lin-etal-2020-autoregressive}, MiniLLM, and GKD\,\cite{agarwal2023gkd}, we initialize the student models with the fine-tuned ones according to their original methods, ensuring a fair comparison in our method as well. However, \alg demonstrates effectiveness even without the need for such fine-tuned student models, unlike other methods that utilize SGOs. The corresponding results are available in Appendix~\ref{app:one_stage}. We conduct validation at the end of every training epoch. For MiniLLM, we have used the original code~\footnote{https://github.com/microsoft/LMOps/tree/main/minillm}~\cite{gu2023knowledge}, while for other baselines, we have re-implemented them. To train the OpenLLaMA2~\cite{openlm2023openllama}, we use LoRA for query and value weights with a rank of 16.

\vspace{-7.5pt}
\paragraph{Task-specific experiments.} For text summarization tasks (SAMSum, XSum, and CNN/DM), we train all teacher and student models for 10 epochs. In contrast, for the IWSLT 2017 En-De dataset, we train them for 2 epochs.
Since the official code for MiniLLM~\cite{gu2023knowledge} is not available on such tasks, we do not consider conducting experiments with MiniLLM on these tasks. Moreover, as the other methods such as SeqKD~\cite{kim-rush-2016-sequence}, ImitKD~\cite{lin-etal-2020-autoregressive}, and GKD~\cite{agarwal2023gkd} highly increase the training time from obtaining the SGOs or teacher-generated outputs, we only use the 20K of random samples for XSum and CNN/DM, as we described in Sec.~\ref{comp:summ}. However, due to the difficulty of machine translation, we use the full dataset of IWSLT 2017. We use a fixed learning rate of 1e-4 and use possible maximum batch size within \{8, 32, 64\} for RTX 3090 GPUs. We also conducted ten validations for all experiments. For the training teacher model, we utilize LoRA~\cite{hu2022lora} for all weights for query, key, value, and output with rank of 16.

\subsection{Evaluation}\label{app:evaluation_details}
For evaluating the teacher and student models, we applied a single A100 40GB GPU for the instruction-following task and a single RTX 3090 GPU for the text summarization and machine translation tasks. 

\vspace{-7.5pt}
\paragraph{Instruction-following.} Our evaluation setup for the instruction-following task also follows the \citet{gu2023knowledge}. During the evaluation, we sample the responses from each model using a temperature of 1.0, a max-length limit of 512, and five random seeds~(\textit{i.e.,} $\{10, 20, 30, 40, 50\}$). We adopt a prompt wrapper as shown in  Fig.~\ref{fig:sample_prompt}. However, for GPT-4 feedback, instead of using the prompt introduced in \citet{gu2023knowledge}, we use a more popular prompt introduced in \citet{zheng2023judging} which is illustrated in Fig.~\ref{fig:gpt4_prompt} with setting the temperature of 0.7. We also report the ratio of the total score of model responses and ground truth answers by following \citet{gu2023knowledge}.

\vspace{-7.5pt}
\paragraph{Task-specific experiments.} During the evaluation, we sample the responses from each model using a greedy sampling, and a max-length limit of 128. We use ROUGE-L \cite{lin-2004-rouge} and BLEU~\cite{papineni-etal-2002-bleu} for text summarization and machine translation, respectively.

\begin{figure}[t]
    \centering
    \small
    \begin{tcolorbox}
    [width=\linewidth, sharp corners=all, colback=gray!10, boxrule=0.3mm]
    Below is an instruction that describes a task. \\
    Write a response that appropriately completes the request. \\

    \#\#\# Instruction: \\
    \{instruction\} \\

    \#\#\# Input: \\
    \{input\} \\

    \#\#\# Response:
    \end{tcolorbox}
    \vspace{-7pt}
    \caption{The prompt template for training and evaluation of instruction-following task experiments from \citet{gu2023knowledge}.}
    \label{fig:sample_prompt}
    \vspace{-10pt}
\end{figure}

\begin{figure}[t]
    \centering
    \small
    \begin{tcolorbox}
    [width=\linewidth, sharp corners=all, colback=gray!10, boxrule=0.3mm]
    [System] \\
    Please act as an impartial judge and evaluate the quality of the response provided by an AI assistant to the user question displayed below. Your evaluation should consider factors such as the helpfulness, relevance, accuracy, depth, creativity, and level of detail of the response. Begin your evaluation by providing a short explanation. Be as objective as possible. After providing your explanation, please rate the response on a scale of 1 to 10 by strictly following this format: ``[[rating]]'', for example: ``Rating: [[5]]''. \\

    [Question] \\
    \{question\} \\

    [The Start of Assistant’s Answer] \\
    \{answer\} 
    
    [The End of Assistant's Answer]
    \end{tcolorbox}
    \vspace{-7pt}
    \caption{The prompt template for single-answer grading of GPT-4 feedback from \citet{zheng2023judging}.}
    \label{fig:gpt4_prompt}
    \vspace{-10pt}
\end{figure}
\section{Additional Results}
In this section, we provide additional experimental results to demonstrate the effectiveness of our proposed method and its components.

\subsection{Full Results of Instruction-Following (Fig.~\ref{fig:instruction_short})}\label{app:instruction}
In Fig.\,\ref{fig:instruction}, we describe the full version of the main result in Sec.\,\ref{comp:inst} and Fig.\,\ref{fig:instruction}. Our proposed \alg consistently outperform the baselines such as MiniLLM~\cite{gu2023knowledge} and GKD~\cite{agarwal2023gkd} in most of the datasets~(Dolly Evaluation, Self-Instruct, Vicuna Evaluation, Super-Natural, and Unnatural) and metrics~(ROUGE-L and GPT-4 feedback) regardless of model sizes. The numerical results are reported in Tab.~\ref{tab:instruction_gpt}, Tab.~\ref{tab:instruction_opt}, and Tab.~\ref{tab:instruction_llama}. These results demonstrate that \alg achieves the best performance~(\textbf{bold number}) in most cases, except for a few second best performances~(\underline{underlined number}).

\subsection{Full Results of Task-Specific KD (Tab.\,\ref{tab:summarization})}\label{app:summarization}
\begin{table}[t]
\centering
\small
\vspace{-5pt}
\caption{The full results of Tab.~\ref{tab:summarization}, which is the performance comparison of KD methods trained on text summarization and machine translation datasets. We report the ROUGE-L and BLEU scores for the distilled student.}
\vspace{5pt}
\resizebox{1.0\columnwidth}{!}{
\addtolength{\tabcolsep}{-3pt}
\begin{tabular}{l|cccc}
\toprule[0.1em]
        Dataset & \!SAMSum\! & XSum & \!CNN/DM\! & IWSLT \\ \midrule
        T5-XL (Teacher) & 52.52 & 30.86 & 40.84 & 34.56 \\ \midrule
        \multicolumn{4}{l}{\textbf{\textit{T5-XL (3B) $\rightarrow$ T5-Small (0.06B)}}} \\ \midrule
        KD~\cite{Hinton2015DistillingTK} & 39.52 & 21.47 & 36.43 & 21.15 \\
        SeqKD~\cite{kim-rush-2016-sequence} & 40.24 & 21.40 & 36.85 & 21.42 \\
        ImitKD~\cite{lin-etal-2020-autoregressive} & 41.44 & 21.96 & 37.68 & 21.52 \\
        GKD~\cite{agarwal2023gkd} & 41.92 & 22.26 & 37.65 & 22.04 \\
        \alg \textbf{(ours)} & \textbf{42.37} & \textbf{22.43} & \textbf{38.01} & \textbf{22.53} \\ \midrule
        \multicolumn{4}{l}{\textbf{\textit{T5-XL (3B) $\rightarrow$ T5-Base (0.2B)}}} \\ \midrule
        KD~\cite{Hinton2015DistillingTK} & 46.23 & 27.15 & 38.49 & 29.36 \\
        SeqKD~\cite{kim-rush-2016-sequence} & 46.89 & 27.35 & 38.65 & 29.07 \\
        ImitKD~\cite{lin-etal-2020-autoregressive} & 48.57 & 28.08 & 39.16 & 29.87 \\
        GKD~\cite{agarwal2023gkd} & 48.49 & 28.15 & 38.98 & 30.24 \\
        \alg \textbf{(ours)} & \textbf{49.11} & \textbf{28.57} & \textbf{39.47} & \textbf{30.32} \\
\bottomrule[0.1em]
\end{tabular}
}\label{tab:translation}
\end{table}
In addition to the results on SAMSum and IWSLT 2017 datasets in Sec.\,\ref{comp:summ} and Tab.\,\ref{tab:summarization}, we further conduct experiments on XSum and CNN/DM datasets. As the existing KD methods using SGO require up to 5$\times$ training time compared to na\"ive KD, we randomly sample 20K from the training dataset. Tab.\,\ref{tab:translation} further shows the performance of XSum and CNN/DM for \alg and other KD baselines. Our \alg consistently outperforms other baselines in terms of ROUGE-L and BLEU scores.

\subsection{Terminal Probability by Adaptive SGO Scheduler}
In the following list, we report the terminal probability values determined by our adaptive SGO scheduler. We observe that the final probability values are varying across the different tasks. This highlights the importance of using our novel SGO scheduler which can adaptively balance the SGOs and a fixed dataset.
\begin{itemize}[leftmargin=*, itemsep=0pt]
    \vspace{-5pt}
    \item \textbf{{GPT-2 family}}\,\texttt{(\small databricks-dolly-15k)}: 0.4 (Base), 0.4 (Medium), 0.5 (Large)
    \item \textbf{{OPT family}} \texttt{(\small databricks-dolly-15k)}: 0.2 (125M), 0.3 (350M), 0.5 (1.3B)
    \item \textbf{{OpenLLaMA2}} \texttt{(\small databricks-dolly-15k)}: 0.6 (3B)
    \item \textbf{{T5}} {\small(SAMSum)}: 0.7 (Small), 0.8 (Base)
    \item \textbf{{T5}} {\small(IWSLT 2017)}: 0.5 (Small), 0.7 (Base)
    \item \textbf{{T5}} {\small(XSum)}: 0.3 (Small), 0.6 (Base)
    \item \textbf{{T5}} {\small(CNN/DM)}: 0.4 (Small), 0.6 (Base)
\end{itemize}

\subsection{Sensitivity Study for $\alpha$}\label{app:sense_alpha}
\begin{table}[t]
\centering
\vspace{-5pt}
\caption{Comparison of the ROUGE-L score of GPT-2 student using different $\alpha$. All results show the concave form in terms of ROUGE-L scores.}
\vspace{5pt}
\resizebox{1.0\columnwidth}{!}{
\addtolength{\tabcolsep}{2pt}
\begin{tabular}{l|c|ccccc}
\toprule[0.1em]
        Loss & $\alpha$ & Dolly & Self & Vicuna & SNI & UNI \\ \midrule
        \multirow{6}{*}{SKL}    & 0.0 & 23.52 & 11.23 & 15.92 & 20.68 & 23.38 \\
                                & 0.1 & 24.80 & 12.86 & 16.20 & 26.43 & 28.06 \\
                                & 0.3 & 24.54 & 12.72 & 16.25 & 25.31 & 27.92 \\
                                & 0.5 & 24.55 & 12.66 & 16.05 & 23.59 & 27.48 \\
                                & 0.7 & 24.49 & 12.08 & 15.86 & 21.49 & 27.15 \\
                                & 0.9 & 24.27 & 11.25 & 14.84 & 19.07 & 26.51 \\ \midrule
        \multirow{6}{*}{SRKL}   & 0.0 & 23.82 & 10.90 & 16.11 & 22.47 & 23.03 \\
                                & 0.1 & 25.21 & 12.98 & 15.77 & 25.83 & 28.62 \\
                                & 0.3 & 25.00 & 12.52 & 15.53 & 25.10 & 27.87 \\
                                & 0.5 & 25.92 & 12.09 & 15.39 & 23.91 & 26.98 \\
                                & 0.7 & 21.41 & 11.73 & 14.69 & 20.65 & 24.16 \\
                                & 0.9 & 18.20 & 9.69 & 13.71 & 17.50 & 19.35 \\
\bottomrule[0.1em]
\end{tabular}
}
\label{tab:alpha}
\vspace{-5pt}
\end{table}
As shown in Tab.~\ref{tab:alpha}, we also report the numerical values of the results in Fig.~\ref{fig:skew} for all datasets. The results for Vicuna and Unnatural are similar to the results for Dolly, Self-Instruct, and Super-Natural Instruction datasets~(\textit{i.e.,} making concave form in terms of ROUGE-L scores as $\alpha$ increases). These results highly relate to our theoretical analysis in Sec.~\ref{sec:skew_kl} that the normalized approximation error dividing by gradient showed the convexity. This also provides the importance of the skew value $\alpha$ which has an important role in balancing the approximation error and gradient scale.

\begin{table}[t]
\centering
\vspace{-5pt}
\caption{Comparison of diverse scheduling strategy of replay ratio.}
\vspace{5pt}
\resizebox{1.0\columnwidth}{!}{
\begin{tabular}{l|c|ccccc}
\toprule[0.1em]
        Replay ratio $\zeta$ & Time & Dolly & Self & Vicuna & SNI & UNI \\ \midrule
        Constant (0.5) & $\times$2.13 & 26.32 & 12.99 & 17.31 & 27.45 & 28.64 \\ 
        Increasing (${t}/{T}$) & $\times$2.39 & \textbf{26.44} & 12.42 & 17.60 & 26.80 & \textbf{29.90} \\
        Decreasing ($1-{t}/{T}$) & $\times$\textbf{1.67} & 26.11 & \textbf{13.14} & \textbf{18.46} & \textbf{27.51} & 29.35 \\
\bottomrule[0.1em]
\end{tabular}
}
\label{tab:replay}
\end{table}
\subsection{Design of Replay Ratio} Tab.~\ref{tab:replay} summarizes the training time and performance according to the different scheduling strategies for replay ratio in the off-policy training. 
Despite the high bias error~\cite{lee2023plastic} of off-policy training, the performance of the constant or increasing $\zeta$ is similarly compared to our proposed decreasing manner thanks to our adaptive probability of using SGO. 
Furthermore, our scheduling strategy not only demonstrates the highest training efficiency but also consistently delivers the best performance overall, achieving the best results in three out of five cases.

\subsection{Synergy of SKL and SRKL}
\begin{figure}[t]
    \centering
    \includegraphics[width=1.0\linewidth]{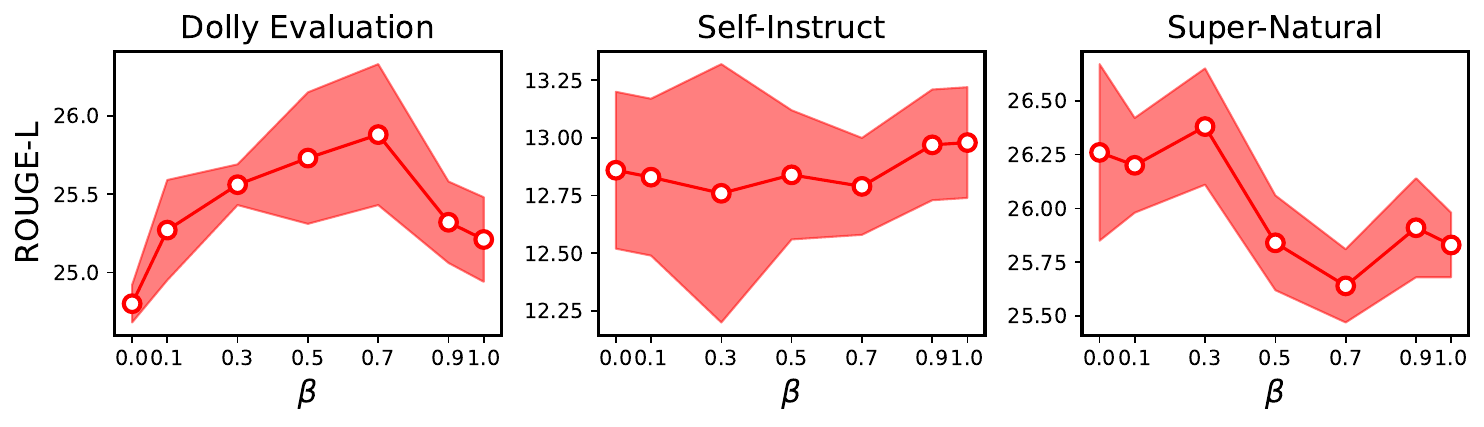}
    \vspace{-25pt}
    \caption{Comparison of the ROUGE-L scores using different $\beta$ for ISKL~(\textit{i.e.,} interpolations between SKL and SRKL).}
    \label{fig:fig_ablation_beta}
    \vspace{-5pt}
\end{figure}
Similar to generalized JSD~\cite{agarwal2023gkd}, we consider the interpolation between the SKL and SRKL using the coefficient $\beta \in [0, 1]$. Here, we define interpolated SKL~(ISKL) as follows:
\begin{equation*}
    D_{\text{ISKL}}^{(\beta)}(p, q_\theta) = \beta D_{\text{SKL}}^{(0.1)}(p, q_\theta) + (1-\beta) D_{\text{SRKL}}^{(0.1)}(p, q_\theta).
\end{equation*}
We set the $\alpha$ for 0.1 based on previously described results. Fig.~\ref{fig:fig_ablation_beta} shows the results of ISKL across different $\beta$ values. Note that we do not use SGO for these results. We observe that there is no consistent tendency between the $\beta$ and corresponding performance. However, the performance for ISKL for all $\beta$ still outperforms the other loss functions previously used in auto-regressive KD~(\textit{i.e.,} JSD and RKLD).


\subsection{Qualitative Evaluation}
We provide some responses generated by the models distilled by different methods based on OpenLLaMA2-3B in Tab.\,\ref{qualitative:dolly} and Tab.\,\ref{qualitative:vicuna}. The prompts are sampled from the \texttt{databricks-dolly-15k}, Self-Inst, and Vicuna. Our results demonstrate that \alg produces more detailed and accurate responses compared to other KD baselines. Specifically, \alg excels at comprehending the precise instructions given (for example, following alphabetical order in Case \#1 from Tab.\,\ref{qualitative:dolly}, and identifying categorization candidates in Case \#2 from Tab.\,\ref{qualitative:vicuna}). Additionally, we observe that \alg achieves high-quality output across various types of tasks, including code generation (as in Case \#1 in Tab.\,\ref{qualitative:vicuna}) and mathematical reasoning (as in Case \#3 in Tab.\,\ref{qualitative:vicuna}).

\subsection{Replay Buffer Capacity}
\begin{table}[t]
\centering
\vspace{-5pt}
\caption{Comparison of the performance of GPT-2 student across different capacities of replay buffer for off-policy approach.}
\vspace{5pt}
\resizebox{1.0\columnwidth}{!}{
\begin{tabular}{l|ccccc}
\toprule[0.1em]
        $\mathcal{D}_{R}$ Capacity & Dolly & Self & Vicuna & SNI & UNI \\ \midrule
        250 & 24.71 & 11.95 & 16.79 & 24.88 & 26.73 \\ 
        500 & 25.32 & 12.46 & 17.64 & 25.69 & 27.00 \\
        1000 & 26.11 & \textbf{13.14} & \textbf{18.46} & \textbf{27.51} & \textbf{29.35} \\
        2000 & \textbf{26.48} & 12.88 & 17.49 & 26.51 & 28.61 \\
        4000 & 25.58 & 12.65 & 17.22 & 25.65 & 27.55 \\
\bottomrule[0.1em]
\end{tabular}
}
\label{tab:capacity}
\end{table}
We also conduct experiments to confirm the effect of the capacity of the replay buffer on our off-policy approach. Table~\ref{tab:capacity} summarizes the performance associated with different capacities of the replay buffer, $\mathcal{D}_{R}$. We observe that determining the appropriate capacity involves a trade-off. A capacity that is too small may lead to overfitting on a limited number of samples within $\mathcal{D}_{R}$. Conversely, a capacity that is too large results in the inclusion of outdated SGO in the replay buffers, potentially introducing a high bias issue as noted by Lee (2023)~\cite{lee2023plastic}. In our experiments, a capacity value of 1000 demonstrates the most balanced performance overall.
\newpage
\begin{table*}[t]
\centering
\vspace{-10pt}
\caption{Comparison with state-of-the-art KD methods, fine-tuned GPT-2 model families~\cite{radford2019language} on \texttt{databricks-dolly-15k} dataset. All results are based on our re-implementation. The \textbf{bold} and \underline{underline} markings indicate the best and second-best results, respectively, among those from the same evaluation dataset and student model.}
\vspace{5pt}
\resizebox{1.0\textwidth}{!}{
\addtolength{\tabcolsep}{2.5pt}
\begin{tabular}{l|cc|cc|cc|c|c}
\toprule[0.1em]
        KD & \multicolumn{2}{c|}{Dolly Evaluation} & \multicolumn{2}{c|}{Self-Instruct} & \multicolumn{2}{c|}{Vicuna Evaluation} & \multicolumn{1}{c|}{Super-Natural} & Unnatural \\ 
        Method & GPT-4 Eval\,($\uparrow$) & ROUGE-L\,($\uparrow$) & GPT-4 Eval\,($\uparrow$) & ROUGE-L\,($\uparrow$) & GPT-4 Eval\,($\uparrow$) & ROUGE-L\,($\uparrow$) & ROUGE-L\,($\uparrow$) & ROUGE-L\,($\uparrow$) \\ \midrule
        GPT-2 XL (Teacher) & 44.91 (0.65) & 26.53 (0.35) & 34.98 (0.92) & 14.41 (0.34) & 33.19 (1.13) & 16.30 (0.40) & 27.76 (0.36) & 32.06 (0.12) \\ \midrule
        
        \multicolumn{9}{l}{\textbf{\textit{GPT-2 XL (1.5B) $\rightarrow$ GPT-2 (0.1B)}}} \\ \midrule
        SFT & 30.26 (0.76) & 23.33 (0.22) & 18.75 (0.55) & 10.56 (0.54) & 18.43 (0.10) & 15.12 (0.47) & 17.08 (0.29) & 20.07 (0.13) \\
        KD\,\cite{Hinton2015DistillingTK} & \underline{30.33 (0.59)} & 23.52 (0.19) & 18.39 (0.70) & 11.23 (0.41) & 18.87 (0.28) & 15.92 (0.37) & 20.68 (0.14) & 23.38 (0.12) \\
        SeqKD\,\cite{kim-rush-2016-sequence} & 30.11 (0.35) & 23.38 (0.37) & 17.94 (0.55) & 10.18 (0.18) & 18.71 (0.67) & 15.01 (0.28) & 15.08 (0.12) & 19.21 (0.06) \\
        ImitKD\,\cite{lin-etal-2020-autoregressive} & 27.91 (0.48) & 21.63 (0.51) & 18.71 (0.42) & 10.85 (0.38) & 19.14 (0.14) & 14.70 (0.29) & 17.94 (0.10) & 21.24 (0.07) \\
        MiniLLM\,\cite{gu2023knowledge} & 29.27 (0.41) & 23.84 (0.26) & \underline{20.76 (0.07)} & 12.44 (0.28) & \underline{20.82 (0.57)} & \underline{18.29 (0.36)} & 22.62 (0.26) & 23.26 (0.09) \\
        GKD\,\cite{agarwal2023gkd} & 29.75 (0.06) & 23.75 (0.15) & 20.34 (0.47) & \underline{12.73 (0.24)} & 20.48 (0.43) & 16.64 (0.24) & \underline{26.05 (0.23)} & \underline{27.70 (0.10)} \\
        \alg & \textbf{31.97 (0.06)} & \textbf{26.11 (0.68)} & \textbf{22.45 (1.29)} & \textbf{13.14 (0.69)} & \textbf{23.21 (1.67)} & \textbf{18.46 (0.53)} & \textbf{27.51 (0.03)} & \textbf{29.35 (0.07)} \\ \midrule
        
        \multicolumn{9}{l}{\textbf{\textit{GPT-2 XL (1.5B) $\rightarrow$ GPT-2 Medium (0.3B)}}} \\ \midrule
        SFT & 38.56 (0.43) & 25.23 (0.36) & 27.48 (0.51) & 13.35 (0.36) & 27.88 (1.21) & 16.17 (0.50) & 23.77 (0.23) & 27.27 (0.08) \\
        KD\,\cite{Hinton2015DistillingTK} & \underline{39.00 (0.77)} & 24.75 (0.54) & \underline{28.17 (0.93)} & 12.84 (0.48) & 28.97 (0.72) & 16.14 (0.37) & 24.00 (0.23) & 27.20 (0.12) \\
        SeqKD\,\cite{kim-rush-2016-sequence} & 38.78 (0.52) & \underline{25.54 (0.34)} & 25.62 (0.76) & 12.69 (0.46) & 29.54 (0.58) & 16.61 (0.44) & 21.79 (0.22) & 26.33 (0.15) \\
        ImitKD\,\cite{lin-etal-2020-autoregressive} & 34.84 (0.11) & 22.60 (0.27) & 24.38 (0.07) & 10.65 (0.48) & 25.90 (0.39) & 15.07 (0.39) & 19.71 (0.20) & 21.81 (0.07) \\
        MiniLLM\,\cite{gu2023knowledge} & 38.68 (0.13) & 25.49 (0.28) & 27.96 (0.57) & \underline{14.30 (0.49)} & 29.00 (1.16) & \underline{18.09 (0.32)} & \underline{26.74 (0.22)} & \underline{30.68 (0.13)} \\
        GKD\,\cite{agarwal2023gkd} & 36.82 (0.79) & 24.51 (0.32) & 27.70 (0.25) & 14.26 (0.29) & \underline{30.25 (1.16)} & 16.86 (0.36) & 26.05 (0.23) & 29.05 (0.11) \\
        \alg & \textbf{39.28 (0.39)} & \textbf{27.62 (0.28)} & \textbf{28.34 (0.29)} & \textbf{15.24 (0.43)} & \textbf{32.05 (1.01)} & \textbf{18.14 (0.28)} & \textbf{29.50 (0.24)} & \textbf{32.71 (0.09)} \\ \midrule
        
        \multicolumn{9}{l}{\textbf{\textit{GPT-2 XL (1.5B) $\rightarrow$ GPT-2 Large (0.8B)}}} \\ \midrule
        SFT & 40.27 (0.36) & 25.19 (0.17) & 28.58 (0.05) & 13.40 (0.35) & 28.07 (0.14) & 16.10 (0.41) & 23.89 (0.23) & 26.82 (0.12) \\
        KD\,\cite{Hinton2015DistillingTK} & 42.03 (1.26) & 26.51 (0.47) & 30.20 (0.90) & 13.99 (0.58) & 31.82 (0.28) & 17.07 (0.20) & 26.48 (0.31) & 30.12 (0.09) \\
        SeqKD\,\cite{kim-rush-2016-sequence} & 40.88 (0.58) & 26.11 (0.31) & 30.79 (1.42) & 14.61 (0.36) & 31.29 (0.32) & 16.23 (0.39) & 25.64 (0.18) & 28.40 (0.06) \\
        ImitKD\,\cite{lin-etal-2020-autoregressive} & 39.60 (0.39) & 24.36 (0.58) & 29.43 (0.76) & 12.48 (0.29) & 28.74 (0.49) & 16.11 (0.39) & 22.84 (0.31) & 26.51 (0.16) \\
        MiniLLM\,\cite{gu2023knowledge} & \textbf{42.62 (0.26)} & 26.45 (0.48) & \underline{31.20 (0.23)} & 15.23 (0.22) & \underline{33.27 (0.09)} & \textbf{18.29 (0.32)} & \underline{29.81 (0.28)} & \underline{33.80 (0.13)} \\
        GKD\,\cite{agarwal2023gkd} & 40.18 (0.77) & 25.52 (0.40) & 29.23 (0.27) & 14.25 (0.25) & 29.86 (0.46) & 16.82 (0.39) & 27.50 (0.16) & 30.48 (0.10) \\
        \alg & \underline{42.34 (1.09)} & \textbf{28.68 (0.33)} & \textbf{32.29 (0.25)} & \textbf{15.85 (0.44)} & \textbf{34.89 (1.29)} & \underline{18.28 (0.33)} & \textbf{31.35 (0.12)} & \textbf{35.23 (0.10)} \\
\bottomrule[0.1em]
\end{tabular}
}\label{tab:instruction_gpt}
\end{table*}

\begin{table*}[t]
\centering
\vspace{-10pt}
\caption{Comparison with state-of-the-art KD methods, fine-tuned OPT model families~\cite{zhang2022opt} on \texttt{databricks-dolly-15k} dataset. All results are based on our re-implementation. The \textbf{bold} and \underline{underline} markings indicate the best and second-best results, respectively, among those from the same evaluation dataset and student model.}
\vspace{5pt}
\resizebox{1.0\textwidth}{!}{
\addtolength{\tabcolsep}{2.5pt}
\begin{tabular}{l|cc|cc|cc|c|c}
\toprule[0.1em]
        KD & \multicolumn{2}{c|}{Dolly Evaluation} & \multicolumn{2}{c|}{Self-Instruct} & \multicolumn{2}{c|}{Vicuna Evaluation} & \multicolumn{1}{c|}{Super-Natural} & Unnatural \\ 
        Method & GPT-4 Eval\,($\uparrow$) & ROUGE-L\,($\uparrow$) & GPT-4 Eval\,($\uparrow$) & ROUGE-L\,($\uparrow$) & GPT-4 Eval\,($\uparrow$) & ROUGE-L\,($\uparrow$) & ROUGE-L\,($\uparrow$) & ROUGE-L\,($\uparrow$) \\ \midrule
        
        OPT-2.7B (Teacher) & 43.15 (0.57) & 26.18 (0.21) & 25.62 (0.96) & 11.19 (0.29) & 30.48 (0.61) & 15.48 (0.37) & 19.19 (0.11) & 22.65 (0.14) \\ \midrule
        
        \multicolumn{9}{l}{\textbf{\textit{OPT-2.7B (2.7B) $\rightarrow$ OPT-125M (0.1B)}}} \\ \midrule
        SFT & 27.56 (0.06) & 21.78 (0.19) & 16.97 (0.54) & 8.09 (0.39) & 18.93 (0.12) & 14.40 (0.17) & 13.45 (0.20) & 15.08 (0.05) \\
        KD\,\cite{Hinton2015DistillingTK} & 26.26 (0.11) & 20.54 (0.38) & \underline{19.91 (0.02)} & 9.16 (0.29) & 18.42 (0.28) & 14.65 (0.47) & 15.79 (0.26) & 18.23 (0.09) \\
        SeqKD\,\cite{kim-rush-2016-sequence} & 26.22 (0.09) & 20.72 (0.59) & 18.55 (0.25) & 8.94 (0.37) & 17.29 (0.62) & 13.56 (0.33) & 16.80 (0.36) & 18.68 (0.16) \\ 
        ImitKD\,\cite{lin-etal-2020-autoregressive} & \underline{27.69 (0.21)} & 20.16 (0.19) & 17.95 (0.25) & 8.95 (0.50) & 18.66 (0.56) & 15.05 (0.52) & 14.72 (0.21) & 16.55 (0.10) \\
        MiniLLM\,\cite{gu2023knowledge} & 26.56 (0.26) & 22.24 (0.32) & 18.73 (0.22) & 9.92 (0.47) & 19.39 (0.08) & \textbf{16.97 (0.49)} & 16.58 (0.19) & 18.59 (0.09) \\
        GKD\,\cite{agarwal2023gkd} & 26.37 (0.24) & \underline{22.46 (0.27)} & 19.87 (0.16) & \underline{10.59 (0.36)} & \underline{20.23 (0.74)} & 16.25 (0.65) & \underline{19.33 (0.20)} & \underline{21.40 (0.16)} \\
        \alg & \textbf{29.30 (0.02)} & \textbf{24.75 (0.31)} & \textbf{20.50 (0.18)} & \textbf{10.93 (0.49)} & \textbf{21.54 (0.12)} & \underline{16.55 (0.38)} & \textbf{24.06 (0.26)} & \textbf{26.02 (0.11)} \\ \midrule
        
        \multicolumn{9}{l}{\textbf{\textit{OPT-2.7B (2.7B) $\rightarrow$ OPT-350M (0.3B)}}} \\ \midrule
        SFT & 31.68 (0.05) & 22.58 (0.38) & 22.52 (0.36) & 11.07 (0.31) & 21.68 (0.43) & 15.10 (0.31) & 19.33 (0.43) & 21.69 (0.04) \\
        KD\,\cite{Hinton2015DistillingTK} & 31.81 (0.10) & 24.01 (0.27) & \underline{22.65 (0.05)} & 11.97 (0.29) & 21.72 (0.19) & 16.12 (0.45) & \underline{22.50 (0.16)} & \underline{25.39 (0.12)} \\
        SeqKD\,\cite{kim-rush-2016-sequence} & 32.13 (0.17) & 24.30 (0.20) & 21.66 (1.58) & 10.69 (0.45) & 22.42 (0.48) & 15.51 (0.18) & 19.93 (0.31) & 22.58 (0.15) \\
        ImitKD\,\cite{lin-etal-2020-autoregressive} & 30.48 (0.31) & 21.77 (0.72) & 22.06 (0.64) & 10.62 (0.28) & 21.30 (0.18) & 15.27 (0.51) & 14.97 (0.20) & 19.12 (0.16) \\
        MiniLLM\,\cite{gu2023knowledge} & 31.95 (0.18) & 24.44 (0.20) & 22.54 (0.23) & \underline{12.41 (0.40)} & \textbf{23.81 (0.32)} & \underline{16.89 (0.34)} & 22.33 (0.17) & 24.20 (0.08) \\
        GKD\,\cite{agarwal2023gkd} & 31.21 (0.63) & 23.39 (0.34) & 21.87 (0.40) & 11.96 (0.54) & 21.46 (0.10) & 16.83 (0.49) & 20.82 (0.24) & 24.18 (0.19) \\
        \alg & \textbf{32.85 (0.81)} & \textbf{26.33 (0.26)} & \textbf{22.70 (0.57)} & \textbf{13.24 (0.29)} & \underline{23.54 (0.69)} & \textbf{17.28 (0.23)} & \textbf{23.95 (0.17)} & \textbf{28.10 (0.11)} \\ \midrule
        
        \multicolumn{9}{l}{\textbf{\textit{OPT-2.7B (2.7B) $\rightarrow$ OPT-1.3B (1.3B)}}} \\ \midrule
        SFT & 38.65 (0.18) & 24.97 (0.33) & 22.06 (0.07) & 13.08 (0.29) & 25.41 (0.10) & 15.52 (0.48) & 24.99 (0.17) & 27.18 (0.16) \\
        KD\,\cite{Hinton2015DistillingTK} & \underline{39.17 (0.28)} & 25.36 (0.35) & 23.99 (0.48) & 13.04 (0.62) & 28.36 (0.88) & 16.21 (0.50) & 25.33 (0.21) & 29.41 (0.13) \\
        SeqKD\,\cite{kim-rush-2016-sequence} & 39.12 (0.46) & \underline{26.26 (0.28)} & \underline{25.01 (0.26)} & 13.15 (0.19) & 28.18 (1.45) & \underline{16.73 (0.52)} & 24.56 (0.17) & 27.76 (0.06) \\ 
        ImitKD\,\cite{lin-etal-2020-autoregressive} & 37.91 (0.57) & 23.92 (0.47) & 21.01 (0.52) & 12.21 (0.59) & 25.11 (0.61) & 16.35 (0.38) & 22.86 (0.28) & 27.33 (0.16) \\
        MiniLLM\,\cite{gu2023knowledge} & 37.27 (0.74) & 24.99 (0.41) & 23.74 (0.91) & 12.45 (0.96) & 28.15 (0.65) & 16.37 (0.17) & 22.86 (0.28) & 27.33 (0.16) \\
        GKD\,\cite{agarwal2023gkd} & 38.93 (1.12) & 26.24 (0.20) & 25.57 (0.33) & \textbf{14.43 (0.75)} & \underline{29.00 (0.50)} & 16.33 (0.37) & \underline{26.86 (0.20)} & \underline{30.62 (0.26)} \\
        \alg & \textbf{41.69 (0.35)} & \textbf{27.30 (0.25)} & \textbf{26.21 (1.09)} & \underline{14.06 (0.64)} & \textbf{29.04 (1.14)} & \textbf{17.27 (0.29)} & \textbf{27.60 (0.16)} & \textbf{31.02 (0.09)} \\
\bottomrule[0.1em]
\end{tabular}
}\label{tab:instruction_opt}
\end{table*}

\begin{table*}[t]
\centering
\caption{Comparison of state-of-the-art KD methods using OpenLLaMA2-7B~\cite{openlm2023openllama} and OpenLLaMA2-3B as teacher and student models, respectively. We fine-tune the models on \texttt{databricks-dolly-15k} dataset. All results are based on our re-implementation. The \textbf{bold} and \underline{underline} markings indicate the best and second-best results, respectively, among those from the same evaluation dataset and student model.}
\vspace{5pt}
\resizebox{1.0\textwidth}{!}{
\addtolength{\tabcolsep}{2.5pt}
\begin{tabular}{l|cc|cc|cc|c|c}
\toprule[0.1em]
        KD & \multicolumn{2}{c|}{Dolly Evaluation} & \multicolumn{2}{c|}{Self-Instruct} & \multicolumn{2}{c|}{Vicuna Evaluation} & \multicolumn{1}{c|}{Super-Natural} & Unnatural \\
        Method & GPT-4 Eval\,($\uparrow$) & ROUGE-L\,($\uparrow$) & GPT-4 Eval\,($\uparrow$) & ROUGE-L\,($\uparrow$) & GPT-4 Eval\,($\uparrow$) & ROUGE-L\,($\uparrow$) & ROUGE-L\,($\uparrow$) & ROUGE-L\,($\uparrow$) \\ \midrule
        
        OpenLLaMA2-7B (Teacher) & 56.86 (0.63) & 27.59 (0.32) & 54.11 (1.38) & 18.99 (0.61) & 46.65 (0.25) & 17.42 (0.57) & 31.26 (0.17) & 31.14 (0.07) \\ \midrule
        
        SFT & 47.27 (0.17) & 25.11 (0.58) & 41.66 (0.28) & 16.52 (0.56) & 35.23 (0.33) & 16.33 (0.39) & 29.28 (0.45) & 29.17 (0.12) \\
        KD\,\cite{Hinton2015DistillingTK} & 44.92 (0.56) & 20.95 (0.49) & 42.13 (0.06) & 16.12 (0.80) & 35.06 (0.17) & 15.39 (0.39) & 27.93 (0.26) & 25.16 (0.22) \\
        SeqKD\,\cite{kim-rush-2016-sequence} & 48.05 (0.64) & 24.67 (0.47) & 45.96 (0.13) & 15.83 (0.62) & 39.57 (1.83) & 17.06 (0.47) & 29.06 (0.28) & 28.56 (0.05) \\
        ImitKD\,\cite{lin-etal-2020-autoregressive} & 52.48 (1.51) & 24.53 (0.26) & 44.14 (1.12) & 17.80 (0.80) & 40.82 (1.08) & 17.36 (0.22) & 31.50 (0.12) & 29.54 (0.13) \\
        MiniLLM\,\cite{gu2023knowledge} & 59.01 (1.64) & 27.88 (0.28) & 52.66 (1.66) & 19.94 (0.58) & 45.33 (0.92) & \textbf{20.50 (0.41)} & 36.91 (0.28) & 36.33 (0.14) \\
        GKD\,\cite{agarwal2023gkd} & 55.86 (0.94) & 26.30 (0.31) & 52.01 (1.03) & 19.56 (1.03) & \underline{46.08 (0.17)} & 18.66 (0.34) & 35.71 (0.20) & 32.57 (0.11) \\
        \alg & \textbf{59.94 (1.19)} & \textbf{29.73 (0.52)} & \textbf{53.29 (0.75)} & \textbf{20.39 (0.56)} & \textbf{49.88 (0.46)} & \underline{19.62 (0.41)} & \textbf{37.64 (0.18)} & \textbf{37.56 (0.08)} \\
\bottomrule[0.1em]
\end{tabular}
}\label{tab:instruction_llama}
\end{table*}

\begin{figure*}[t]
    \centering
    \includegraphics[width=1.0\linewidth]{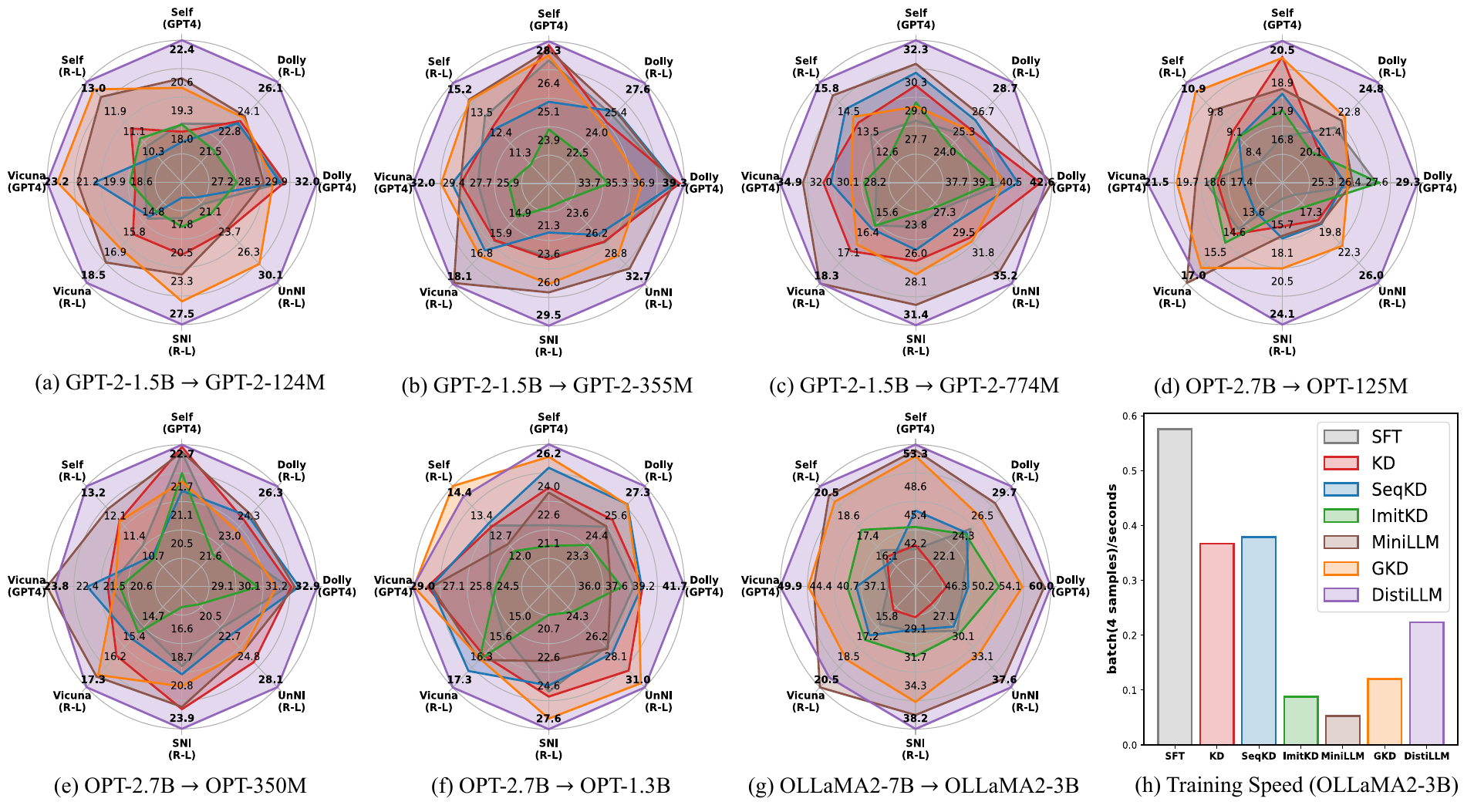}
    \vspace{-17.5pt}
    \caption{Instruction-following tasks, distilling GPT-2~\cite{radford2019language}, OPT~\cite{zhang2022opt}, and OpenLLaMA~(OLLaMA; \citealt{openlm2023openllama}) model families on \texttt{databricks-dolly-15k} dataset. GPT4 and R-L indicate GPT-4 feedback~\cite{zheng2023judging} and ROUGE-L~\cite{lin-2004-rouge}, respectively. To evaluate training speed, results are obtained using four A100 GPUs. In SeqKD~\cite{kim-rush-2016-sequence}, we also consider the generation time for teacher-generated data.}
    \label{fig:instruction}
\end{figure*}
\newpage
\newcommand*\colourcmark[1]{%
  \expandafter\newcommand\csname #1cmark\endcsname{\textcolor{#1}{\ding{51}}}%
}
\newcommand*\colourxmark[1]{%
  \expandafter\newcommand\csname #1xmark\endcsname{\textcolor{#1}{\ding{55}}}%
}
\newcommand*\colourtmark[1]{%
  \expandafter\newcommand\csname #1tmark\endcsname{\textcolor{#1}{\ding{115}}}%
}
\colourcmark{pinegreen}
\colourxmark{red}
\colourtmark{orange}

\begin{table*}[t]
\centering
\vspace{-10pt}
\caption{Instruction-following cases from the \texttt{databricks-dolly-15k}~\cite{DatabricksBlog2023DollyV2} dataset.}
\vspace{5pt}
\resizebox{1.0\textwidth}{!}{
\addtolength{\tabcolsep}{2.5pt}
\begin{tabular}{m{3cm}m{20cm}}
\toprule[0.1em]
        \multicolumn{2}{c}{\textbf{\textit{Case \# 1}}} \\ \midrule[0.1em]
        \rowcolor{Gainsboro!50}
        Instruction & Extract the names of the cities and rank them in alphabetical order? \\ \midrule
        \rowcolor{Gainsboro!50}
        Input & The United Kingdom is a constitutional monarchy and parliamentary democracy.[m][25] Its capital and largest city is London, the capital of England, a global city and financial centre with a population of over 14 million people. Edinburgh, Cardiff and Belfast are the national capitals of Scotland, Wales and Northern Ireland respectively. \\ \midrule
        \rowcolor{Gainsboro!50}
        Ground-truth & The cities listed in alphabetical order are Belfast, Cardiff, Edinburgh and London. \\
        \midrule \midrule
        SFT \redxmark & Alphabetical order of cities in the United Kingdom: \textbackslash n\textbackslash n London, Edinburgh, Cardiff, Belfast, Glasgow, Southampton, Birmingham, Coventry, Leicester, Norwich, Swindon, Cambridge \\
        \midrule
        KD \orangetmark & London, Edinburgh, Cardiff, Belfast \\
        \midrule
        SeqKD \redxmark & Edinburgh, Cardiff, Belfast are the national capital of Scotland, Wales and North Irelandrespectively. \\
        \midrule
        ImitKD \orangetmark & London.Edinburgh.Cardiff.Belfast. \\
        \midrule
        MiniLLM \orangetmark & London, Edinburgh, Cardiff, Belfast \\
        \midrule
        GKD \orangetmark & London, Edinburgh, Cardiff, Belfast \\
        \midrule
        \alg \pinegreencmark & The following are the names of the cities in alphabetical order: \textbackslash n \textbackslash n Belfast\textbackslash n Cardiff\textbackslash n Edinburgh\textbackslash n London \\ \midrule[0.1em]
        \multicolumn{2}{c}{\textbf{\textit{Case \# 2}}} \\ \midrule[0.1em]
        \rowcolor{Gainsboro!50}
        Instruction &  Extract all of the dates mentioned in this paragraph and list them using bullets in the format {Date} - {Description} \\ \midrule
        \rowcolor{Gainsboro!50}
        Input & Daniel Leavitt was born November 16, 1813, at Rye, New Hampshire, the son of Benning Leavitt, influential businessman, state senator, county commissioner and later Chicopee Selectman,[5] and his wife Olive (Jenness) Leavitt. Daniel Leavitt married in 1838 at West Springfield, Massachusetts, Ruth Jeannette Ball. They had three children. \\ \midrule
        \rowcolor{Gainsboro!50}
        Ground-truth & • November 16, 1813 - This is when Daniel Leavitt was born \textbackslash n • 1838 - This is when Daniel Leavitt married his wife Ruth Jeannette Ball. \\
        \midrule \midrule
        SFT \redxmark & • November 16,1813 - Daniel Leavitt was born\textbackslash n• November 17, 1814 - His mother Olive Leavitt died\textbackslash n • November 1, 1815 - His father Benning Leavitt died\textbackslash n * December 24, 1818 - Daniel and Ruth Jeannette Ball were married\textbackslash n * March 31, 1838 - They got married in March 31,1838\textbackslash n * May 21, 1840 - Their third child was born\textbackslash n * May 31, 2018 - Daniel and his children restored their home Rye\textbackslash n * November 1, 2021 - Daniel Leavitt died \\ 
        \midrule
        KD \redxmark & - November 16, \\
        \midrule
        SeqKD \redxmark &  • November 16,   2013 - Daniel Leavitt was born on this date.\textbackslash n • September 27, 1838 - Daniel married at West Springfield, Massachusetts.\\
        \midrule
        ImitKD \orangetmark & - November 16,1813 - Day when he was born\textbackslash n - West Springfield, Massachusetts - Marriage ceremony\textbackslash n - September 5,1857 - Year he died \\
        \midrule
        MiniLLM \orangetmark & • November 16,1813 - Daniel Leavitt was born at Rye, New Haven, the son of Bening Leavitt, influential Businessman, state senator and County Commissioner, and his wife Olive (jenness) Leavit.\textbackslash n • November 19,1838 - Daniel Leavit married at West Springfield, Massachusetts.\textbackslash n • November 23,1913 - Daniel Leavy died at Chicopee, Massachusetts. \\
        \midrule
        GKD \orangetmark & • November 16,1813 - at Rye, New hampshire\textbackslash n • July 2,1838 - at West Springfield, Massachusetts \\
        \midrule
        \alg \pinegreencmark &  • November 16,1813 - Daniel Leavitt was born at Rye, New Hampsire\textbackslash n • November 1838 - Daniel Leavitt married Ruth Jeannette Ball \\ \midrule[0.1em]
        \multicolumn{2}{c}{\textbf{\textit{Case \# 3}}} \\ \midrule[0.1em]
        \rowcolor{Gainsboro!50}
        Instruction &  Which Dutch artist painted “Girl with a Pearl Earring”? \\ \midrule
        \rowcolor{Gainsboro!50}
        Input &  \\ \midrule
        \rowcolor{Gainsboro!50}
        Ground-truth & Vermeer \\
        \midrule  \midrule
        SFT \redxmark & Dimitri Heidelbauer \\
        \midrule
        KD \redxmark & Malahema \\
        \midrule
        SeqKD \redxmark & Dutch artist, Mondrian, painted the Girl with a Pearl Earring. \\
        \midrule
        ImitKD \redxmark & Malahema Vanpolder \\
        \midrule
        MiniLLM \redxmark & Vincent Van Gogh \\
        \midrule
        GKD \orangetmark & Grachtengordel, which is commonly known as “Girl with a pearl earring” is an oil painting by Johannes Vermeer. It was first exhibited in June 1665 in the Hofje van Ca’ra van Oostende, which was in the city of 5de Oostende. \\
        \midrule
        \alg \pinegreencmark &  Johannes Vermeer \\
\bottomrule[0.1em]
\end{tabular}
}\label{qualitative:dolly}
\end{table*}

\newpage
\begin{table*}[t]
\centering
\vspace{-10pt}
\caption{Instruction-following cases from the Self-Instruct~\cite{wang-etal-2023-self-instruct} and Vicuna~\cite{vicuna2023}. Note that the ground-truth for case \#1 is for Python, not for the bash command. However, only \alg generates the simple correct responses.}
\vspace{5pt}
\resizebox{1.0\textwidth}{!}{
\addtolength{\tabcolsep}{2.5pt}
\begin{tabular}{m{3cm}|m{20cm}}
\toprule[0.1em]
              
        \multicolumn{2}{c}{\textbf{\textit{Case \# 1 (Self-Instruct)}}} \\ \midrule[0.1em]
        \rowcolor{Gainsboro!50}
        Instruction & Convert the given description to a bash command. \\ \midrule
        \rowcolor{Gainsboro!50}
        Input & read the first 10 lines from f \\ \midrule
        \rowcolor{Gainsboro!50}
        Ground-truth & f.readlines(10) \\
        \midrule \midrule
        SFT \orangetmark & tail -n 10 f \\
        \midrule
        KD \redxmark & f -n 10 \\
        \midrule
        SeqKD \redxmark & cat f \\
        \midrule
        ImitKD \redxmark & cat f | nl 10 \textbackslash n \textbackslash n \textbackslash n Writing data into the file \textbackslash n \textbackslash n \textbackslash t Christina writes the first 10 documents to back t \textbackslash n \textbackslash n 10 documents written \textbackslash n \textbackslash n Christina reads the first 10 files from t \textbackslash n \textbackslash n Christina writes the first ten ABC documents to t rw ec \textbackslash n \textbackslash n 10 documents read \\
        \midrule
        MiniLLM \pinegreencmark & cat f | head -10 \\
        \midrule
        GKD \orangetmark & echo -e '\textbackslash n read -- we read from stdin\textbackslash n f -- the file\textbackslash n ' | read -n 10 f \\
        \midrule
        \alg \pinegreencmark & head -n 10 f \\ \midrule[0.1em]
        
        \multicolumn{2}{c}{\textbf{\textit{Case \# 2 (Self-Instruct)}}} \\ \midrule[0.1em]
        \rowcolor{Gainsboro!50}
        Instruction & Categorize the Business into one of the Restaurants, Home Services, Auto Services and Miscellaneous based on its Specialties. \\ \midrule
        \rowcolor{Gainsboro!50}
        Input & Call us at 650-636-4884 or visit our website to receive a quote. This shop specializes in New Tires and General Auto Repair. We carry all tires in-house and have a wide range to fit any budget or car specifics. If you are unsure what tires you need there are experts on hand and available to help you pick the best tire for your needs. We also carry commercial vehicle tires and can serve a wide array of fleets. \\ \midrule
        \rowcolor{Gainsboro!50}
        Ground-truth & Auto Services \\
        \midrule \midrule
        SFT \redxmark & This shop specializes in New Tases and General Auto Repair. \\
        \midrule
        KD \redxmark & Home Services, Auto Services and restaurant will be Restaurant because it focuses on Restaurants while automotive store and general auto repair are both services and restaurant would therefore not belong in either category. \\
        \midrule
        SeqKD \orangetmark & The types of businesses listed here are Auto Services, Home Services and Miscellaneous. Because this business has multiple types it is considered a Mischangeluz. \\
        \midrule
        ImitKD \redxmark & Restaurants - Home Services - Auto Services - Miscellaneous \\
        \midrule
        MiniLLM \redxmark & Restaurants, Home Services, AutoServices and Miscellaneous \\
        \midrule
        GKD \redxmark & Options: Restaurants, Home Services, Car Services, Miscellaneous \\
        \midrule
        \alg \pinegreencmark & Specialties: Auto Services \\
        \midrule[0.1em]
        
        \multicolumn{2}{c}{\textbf{\textit{Case \# 3 (Vicuna)}}} \\ \midrule[0.1em]
        \rowcolor{Gainsboro!50}
        Instruction & Given that f(x) = 5x\^{}3  - 2x + 3, find the value of f(2). \\ \midrule
        \rowcolor{Gainsboro!50}
        Input &  \\ \midrule
        \rowcolor{Gainsboro!50}
        Ground-truth & To find the value of f(2), we simply substitute 2 in place of x in the given equation of f(x): \textbackslash n \textbackslash n f(2) = 5(2)\^{}3 - 2(2) + 3 \textbackslash n \textbackslash n f(2) = 5(8) - 4 + 3 \textbackslash n \textbackslash n f(2) = 40 - 1 \textbackslash n \textbackslash n Therefore, the value of f(2) is 39.\\
        \midrule \midrule
        SFT \redxmark & We know that \textbackslash(f(2) = 5(2)\^{}3 - 2(2) + 3 = 1028 - 2 = 1021, \textbackslash hospital doors 5 meters apart. The width of the hospital doors should be \textbackslash(\textbackslash frac\{1021\}\{5\}\textbackslash) = 202.8 inches. At least 202 hospital doors should be narrow enough to fit the estimated population of f(2)iashti2. \\
        \midrule
        KD \redxmark & Given f(x) = (5x\^{}3)\^{}–3 – 2x\^{}3 – 3,  when f(2) is given, substitute the value of x to derive f(2) back. In the given expression,\textbackslash n f(2) is obtained by saying multiplication of 5 and 2 over x\^{}2 so that\textbackslash n f(x) = (x\^{}2 * 5)\^{}–3 – 3 (adding significance over the parenthesis)\textbackslash n So, f(2) is obtained as\textbackslash n \textbackslash t f(2) = (2 * 5)\^{}(–3) – 3 \textbackslash n So, final answer for f(2) is \textbackslash n \textbackslash t f(2){-eq} (2 * 5)\^{}{(-3)} – 3 \textbackslash n Being little bit educated, it can be sequenced solution chain also. So, \textbackslash n \textbackslash t f(2)\textbackslash\{= 5\*\textbackslash\} ×(-3) \textbackslash+ 2 ×[-3]3 * 2 $\alpha$ ⁄ 2 \textbackslash n or \textbackslash n \textbackslash t f(2)= 6 -2 + 3 \textbackslash n Finding the value of f(x) is a linear inversion case. \\
        \midrule
        SeqKD \redxmark & We can substitute for x in f(x) to get f(2) = 5(2 squared) + 2 = 10 and then f(10) = 5\^{}3 + 2 = 25, which is not an integer. Thus, f(2) cannot be an integer, and therefore there isn't a single positive integer that corresponds to f(2). \\
        \midrule
        ImitKD \redxmark & The polynomial f(x) can be rewritten as:\textbackslash n f(x) = 2x\^{}4 + 5x\^{}2 + 3.\textbackslash n \textbackslash n By straightforward algebraic manipulation, we can see that f(2) = -1 \textbackslash n therefore f(2) = 2 + 5 * 2 + 3 = -1 \\
        \midrule
        MiniLLM \redxmark & f(2) = 2(5\^{}3 - 2) + 3 \\
        \midrule
        GKD \redxmark & The polynomial f(x) can be rewritten as:\textbackslash n 5x\^{}3 -2x + 3 = (5x)\^{}2 - 2(5x) + 3 - (2x)\^{}2 + 3\textbackslash n =>5(x\^{}2 - 2) + 3 = 0\textbackslash n =>x\^{}2 - 6 = 0  \\
        \midrule
        \alg \pinegreencmark & f(2) = 5(2)\^{}3 - 2(2) + 3 = 40 - 4 + 3 = 39 \\
\bottomrule[0.1em]
\end{tabular}
}\label{qualitative:vicuna}
\end{table*}

\end{document}